\documentclass{article}

\usepackage{microtype}
\usepackage{graphicx}
\usepackage{subfigure}
\usepackage{booktabs}
\usepackage{hyperref}

\usepackage[accepted]{icml2023}

\usepackage{amsmath}
\usepackage{amssymb}
\usepackage{mathtools}
\usepackage{amsthm}

\usepackage[capitalize,noabbrev]{cleveref}

\newcommand{\argmin}{\mathop{\mathrm{argmin}}}
\newcommand{\argmax}{\mathop{\mathrm{argmax}}}

\def\R{\mathbb{R}}
\def\E{\mathbb{E}}

\def\1{\mathbbm{1}}

\def\half{\frac{1}{2}}

\def\tr{\mathrm{tr}}

\def\cA{\mathcal{A}}

\def\cF{\mathcal{F}}

\def\cH{\mathcal{H}}

\def\cU{\mathcal{U}}

\def\cW{\mathcal{W}}
\def\cX{\mathcal{X}}
\def\cY{\mathcal{Y}}

\theoremstyle{plain}
\newtheorem{theorem}{Theorem}[section]

\newtheorem{lemma}[theorem]{Lemma}

\theoremstyle{definition}

\newtheorem{assumption}[theorem]{Assumption}
\theoremstyle{remark}
\newtheorem{remark}[theorem]{Remark}

\usepackage[textsize=tiny]{todonotes}

\icmltitlerunning{Global Optimization with Parametric Function Approximation}

\begin{document}

\twocolumn[
\icmltitle{Global Optimization with Parametric Function Approximation}

\icmlsetsymbol{equal}{*}

\begin{icmlauthorlist}
\icmlauthor{Chong Liu}{ucsb}
\icmlauthor{Yu-Xiang Wang}{ucsb}
\end{icmlauthorlist}

\icmlaffiliation{ucsb}{Department of Computer Science, University of California, Santa Barbara, CA 93106, USA}

\icmlcorrespondingauthor{Chong Liu}{chongliu@cs.ucsb.edu}
\icmlcorrespondingauthor{Yu-Xiang Wang}{yuxiangw@cs.ucsb.edu}

\icmlkeywords{Global optimization, parametric function approximation, bandit}

\vskip 0.3in
]

\printAffiliationsAndNotice{}

\begin{abstract}
We consider the problem of global optimization with noisy zeroth order oracles — a well-motivated problem useful for various applications ranging from hyper-parameter tuning for deep learning to new material design. Existing work relies on Gaussian processes or other non-parametric family, which suffers from the curse of dimensionality. In this paper, we propose a new algorithm GO-UCB that leverages a parametric family of functions (e.g., neural networks) instead. Under a realizable assumption and a few other mild geometric conditions, we show that GO-UCB achieves a cumulative regret of $\tilde{O}(\sqrt{T})$ where $T$ is the time horizon. At the core of GO-UCB is a carefully designed uncertainty set over parameters based on gradients that allows optimistic exploration. Synthetic and real-world experiments illustrate GO-UCB works better than popular Bayesian optimization approaches, even if the model is misspecified.
\end{abstract}

\section{Introduction}\label{sec:intro}
We consider the problem of finding a global optimal solution to the following optimization problem 
\begin{align*}
\max_{x \in \cX} f(x),
\end{align*} 
where $f: \cX \rightarrow \mathbb{R}$ is an unknown non-convex function that is not necessarily differentiable in $x$.

This problem is well-motivated by many real-world applications. 
For example, the accuracy of a trained neural network on a validation set is complex non-convex function of a set of hyper-parameters (e.g., learning rate, momentum, weight decay, dropout, depth, width, choice of activation functions ...) that one needs to maximize \citep{kandasamy2020tuning}.  
Also in material design, researchers want to synthesize ceramic materials, e.g., titanium dioxide ($\mathrm{TiO}_2$) thin films, using microwave radiation \citep{nakamura2017design} where the film property is a non-convex function of parameters including temperature, solution concentration, pressure, and processing time. Efficiently solving such non-convex optimization problems could significantly reduce energy cost.

We assume having access to only noisy function evaluations, i.e., at round $t$, we select a point $x_t$ and receive a noisy function value $y_t$,
\begin{align}
y_t = f(x_t) + \eta_t,\label{eq:y}
\end{align}
where $\eta_t$ for $t=1,...,T$ are \emph{independent}, \emph{zero-mean}, $\sigma$-\emph{sub-Gaussian} noise. This is known as the \emph{noisy zeroth-order oracle} setting in optimization literature. Let $f^*$ be the optimal function value, following the tradition of Bayesian optimization (see e.g., \citet{frazier2018tutorial} for a review), throughout this paper, we use \emph{cumulative regret} as the evaluation criterion, defined as 
\begin{align*}
R_T = \sum_{t=1}^T r_t = \sum_{t=1}^T f^* - f(x_t),
\end{align*}
where $r_t$ is called instantaneous regret at round $t$. 
An algorithm $\cA$ is said to be a no-regret algorithm if $\lim_{T \rightarrow \infty} R_T(\cA)/T=0$.

Generally speaking, solving a global non-convex optimization is NP-hard \citep{jain2017non} and we need additional assumptions to efficiently proceed. Bayesian optimization usually assumes the objective function $f$ is drawn from a Gaussian process prior.
\citet{srinivas2010gaussian} proposed the GP-UCB approach, which iteratively queries the argmax of an upper confidence bound of the current posterior belief, before updating the posterior belief using the new data point. However, Gaussian process relies on kernels, e.g., squared error kernel or Mat\'ern kernel, which suffer from the curse of dimensionality. A folklore rule-of-thumb is that GP-UCB becomes unwieldy when the dimension is larger than $10$.

A naive approach is to passively query $T$ data points uniformly at random, estimate $f$ by $\hat{f}$ using supervised learning, then return the maximizer of the plug-in estimator $\hat{x} =\argmax_{x \in \cX}\hat{f}(x)$. This may side-step the curse-of-dimensionality depending on which supervised learning model we use.
The drawback of this passive query model is that it does not consider the structure of the function nor does it quickly ``zoom-in'' to the region of the space that is nearly optimal. In contrast, an active query model allows the algorithm to iteratively interact with the function. At round $t$, the model collects information from all previous rounds $1,...,t-1$ and decides where to query next. 

\textbf{GO-UCB Algorithm.} In this paper, we develop an algorithm that allows Bayesian optimization-style active queries to work for general supervised learning-based function approximation. We assume that the supervised learning model $f_w:\cX \rightarrow \R$ is differentiable w.r.t. its $d_w$-dimensional parameter vector $w\in \cW\subset \R^{d_w}$ and that the function class $\cF=\{f_w | w\in\cW\}$ is flexible enough such that the true objective function $f = f_{w^*}$ for some $w^*\in\cW$, i.e., $\cF$ is  \emph{realizable}.  Our algorithm --- \emph{Global Optimization via Upper Confidence Bound} (GO-UCB) --- has two phases:

\fbox{\parbox{0.97\columnwidth}{The \textit{GO-UCB} Framework:
\small
\noindent
\begin{itemize}
    \item Phase I: Uniformly explore $n$ data points.
    \item Phase II: Optimistically explore $T$ data points.
    \vspace{-1em}
\end{itemize}}}

The goal of Phase I to sufficiently explore the function and make sure the estimated parameter $\hat{w}_0$ is close enough to true parameter $w^*$ such that exploration in Phase II are efficient. To solve the estimation problem, we rely on a regression oracle that is able to return an estimated $\hat{w}_0$ after $n$ observations. In details, after Phase I we have a dataset $\{(x_j,y_j)\}_{j=1}^{n}$, then
\begin{align}
\hat{w}_0 \leftarrow \argmin_{w\in\cW} \sum_{j=1}^n (f_w(x_j) - y_j)^2.\label{eq:regression_oracle}
\end{align}
This problem is known as a \emph{non-linear least square} problem. It is computationally hard in the worst-case, but many algorithms are known (e.g., SGD, Gauss-Newton, Levenberg-Marquardt) to effectively solve this problem in practice. Our theoretical analysis of $\hat{w}_0$ uses techniques from \citet{nowak2007complexity}. See Section \ref{sec:mle} for details.

In Phase II, exploration is conducted following the principle of ``Optimism in the Face of Uncertainty'', i.e., the parameter is optimized within an uncertainty region that always contains the true parameter $w^*$. Existing work in bandit algorithms provides techniques that work when $f_w$ is a linear function \citep{abbasi2011improved} or a generalized linear function \citep{li2017provably}, but no solution to general differentiable function is known. 
At the core of our GO-UCB is a carefully designed uncertainty ball $\mathrm{Ball}_t$ over parameters based on gradients, which allows techniques from the linear bandit \citep{abbasi2011improved} to be adapted for the non-linear case. In detail, the ball is defined to be centered at $\hat{w}_t$ --- the solution to a regularized online regression problem after $t-1$ rounds of observations. And the radius of the ball is measured by the covariance matrix of the gradient vectors of all previous rounds. We prove that $w^*$ is always trapped within the ball with high probability.

\textbf{Contributions.} In summary, our main contributions are:
\begin{enumerate}
\item We initiate the study of global optimization problem with parametric function approximation and proposed a new optimistic exploration algorithm --- GO-UCB.
\item Assuming \emph{realizability} and other mild geometric conditions, we prove that GO-UCB converges to the global optima with cumulative regret at the order of $\tilde{O}(\sqrt{T})$ where $T$ is the time horizon.
\item GO-UCB does not suffer from the curse of dimensionality like Gaussian processes-based Bayesian optimization methods. The unknown objective function $f$ can be high-dimensional, non-convex, non-differentiable, and even discontinuous in its input domain. 
\item Synthetic test function and real-world hyperparameter tuning experiments show that GO-UCB works better than all compared Bayesian optimization methods in both realizable and misspecified settings.
\end{enumerate}

\textbf{Technical novelties.} The design of GO-UCB algorithm builds upon the work of \citet{abbasi2011improved} and \citet{agarwal2021rl}, but requires substantial technical novelties as we handle a generic nonlinear parametric function approximation. Specifically:
\begin{enumerate}
    \item LinUCB analysis (e.g., self-normalized Martingale concentration, elliptical potential lemmas \citep{abbasi2011improved,agarwal2021rl}) is not applicable for nonlinear function approximation, but we showed that they can be adapted for this purpose if we can \emph{localize} the learner to a neighborhood of $w^*$. 
    \item We identify a new set of structural assumptions under which we can localize the learner sufficiently with only $O(\sqrt{T})$ rounds of pure exploration.
    \item Showing that $w^*$ remains inside the parameter uncertainty ball $\mathrm{Ball}_t, \forall t \in [T]$ is challenging. We solve this problem by setting regularization centered at the initialization parameter $\hat{w}_0$ and presenting novel inductive proof of a lemma showing $\forall t\in [T], \hat{w}_t$ converges to $w^*$ in $\ell_2$-distance at the same rate.
\end{enumerate}
These new techniques could be of independent interest.

\section{Related Work}\label{sec:rw}
Global non-convex optimization is an important problem that can be found in a lot of research communities and real-world applications, e.g., optimization \citep{rinnooy1987stochastic1,rinnooy1987stochastic2}, machine learning \citep{bubeck2011x,malherbe2017global}, hyperparameter tuning \citep{hazan2018hyperparameter}, neural architecture search \citep{kandasamy2018neural,wang2020learning}, and material discovery \citep{frazier2016bayesian}.

One of the most prominent approaches to this problem is Bayesian Optimization (BO) \citep{shahriari2015taking}, in which the objective function is usually modeled by a Gaussian Process (GP) \citep{williams2006gaussian}, so that the uncertainty can be updated under the Bayesian formalism. Among the many notable algorithms in GP-based BO \citep{srinivas2010gaussian,jones1998efficient,bull2011convergence,frazier2009knowledge,agrawal2013thompson,cai2021on},
GP-UCB \citep{srinivas2010gaussian} is the closest to our paper because our algorithm also selects data points in a UCB (upper confidence bound) style but the construction of the UCB in our paper is different since we are not working with GPs. \citet{scarlett2017lower} proves lower bounds on regret for noisy Gaussian process bandit optimization. GPs are highly flexible and can approximate any smooth functions, but such flexibility comes at a price to play --- curse of dimensionality. Most BO algorithms do not work well when $d>10$. Notable exceptions include the work of
\citet{shekhar2018gaussian,calandriello2019gaussian,eriksson2019scalable,salgia2021domain,rando2022ada} who designed more specialized BO algorithms for high-dimensional tasks. 

Besides BO with GPs, other nonparametric families were considered for global optimization tasks, but they, too, suffer from the curse of dimensionality. We refer readers to 
\citet{wang2018optimization} and the references therein.

While most BO methods use GP as surrogate models, there are other BO methods that use alternative function classes such as neural networks \citep{snoek2015scalable,springenberg2016bayesian}. These methods are different from us in that they use different ways to fit the neural networks and a Monte Carlo sampling approach to decide where to explore next. Empirically, it was reported that they do not outperform advanced GP-based methods that use trust regions \citep{eriksson2019scalable}.

Our problem is also connected to the bandits literature \citep{li2019nearly,foster2020beyond,russo2013eluder,filippi2010parametric}. The global optimization problem can be written as a nonlinear bandits problem in which queried points are actions and the function evaluations are rewards. However, no bandits algorithms can simultaneously handle an infinite action space and a generic nonlinear reward function. Here ``generic'' means the reward function is much more general than a linear or generalized linear function \citep{filippi2010parametric}. 
To the best of our knowledge, we are the first to address the infinite-armed bandit problems with a general differentiable value function (albeit with some additional assumptions).

A recent line of work studied bandits and global optimization with neural function approximation \citep{zhou2020neural,zhang2020neural,dai2022sample}. The main difference from us is that these results still rely on Gaussian processes with a Neural Tangent Kernel in their analysis, thus intrinsically linear. Their regret bounds also require the width of the neural network to be much larger than the number of samples to be sublinear. In contrast, our results apply to general nonlinear function approximations and do not require overparameterization.

\section{Preliminaries}\label{sec:pre}

\subsection{Notations}
We use $[n]$ to denote the set $\{1,2,...,n\}$. The algorithm queries $n$ points in Phase I and $T$ points in Phase II. Let $\cX \subset \mathbb{R}^{d_x}$ and $\cY \subset \mathbb{R}$ denote the domain and  range of $f$, and $\cW \subset [0, 1]^{d_w}$ denote the parameter space of a family of functions $\cF := \{f_w: \cX\rightarrow \cY |  w\in \cW\}$. For convenience, we denote the bivariate function $f_w(x)$ by $f_x(w)$ when $w$ is the variable of interest. $\nabla f_x(w)$ and $\nabla^2 f_x(w)$ denote the gradient and Hessian of function $f$ w.r.t. $w$. $L(w) := \E_{x \sim \cU} (f_x(w)-f_x(w^*))^2$ denotes the (expected) risk function where $\cU$ is uniform distribution.
For a vector $x$, its $\ell_p$ norm is denoted by $\|x\|_p = (\sum_{i=1}^d |x_i|^p)^{1/p}$ for $1\leq p < \infty$ and its $\ell_\infty$ norm is denoted by $\|x\|_\infty = \max_{i \in [d_x]} |x_i|$. 
For a matrix $A$, its operator norm is denoted by $\|A\|_\mathrm{op}$. 
For a vector $x$ and a square matrix $A$, define $\|x\|^2_A = x^\top A x$. Throughout this paper, we use standard big $O$ notation that hide universal constants; and to improve the readability, we use $\tilde{O}$ to hide all logarithmic factors as well as all polynomial factors in problem-specific parameters except $d_w, 1/\mu, T$. For reader's easy reference, we list all symbols and notations in Appendix \ref{sec:table}.

\subsection{Assumptions}
Here we list main assumptions that we will work with throughout this paper. The first assumption says that we have access to a differentiable function family that contains the unknown objective function.

\begin{assumption}[Realizability]\label{ass:parameter_class}
There exists $w^*\in \cW$ such that the unknown objective function $f= f_{w^*}$.
Also, assume $\cW\subset [0,1]^{d_w}$. This is w.l.o.g. for any compact $\cW$.
\end{assumption}
Realizable parameter class is a common assumption in literature \citep{chu2011contextual,foster2018practical,foster2020beyond}, usually the starting point of a line of research for a new problem because one doesn't need to worry about extra regret incurred by misspecified parameter. Although in this paper we only theoretically study the realizable parameter class, our GO-UCB algorithm empirically works well in misspecified tasks too. 

The second assumption is on properties of the function approximation.
\begin{assumption}[Bounded, differentiable and smooth function approximation]\label{ass:objective}
There exist constants $F, C_g, C_h > 0$ such that
$\forall x \in \cX, \forall w \in \cW$, it holds that $|f_x(w)|\leq F,$
\begin{align*}
\|\nabla f_x(w) \|_2 \leq C_g, \quad\text{ and }\quad
\|\nabla^2 f_x(w)\|_\mathrm{op} \leq C_h.
\end{align*} 
\end{assumption}
This assumption imposes mild regularity conditions on the smoothness of the function with respect to its parameter vector $w$. 

The third assumption is on the expected loss function over the uniform distribution (or any other exploration distribution) in the Phase I of GO-UCB.
\begin{assumption}[Geometric conditions on the loss function]\label{ass:loss}
$L(w)=\E_{x\sim \cU}(f_x(w)-f_x(w^*))^2$ satisfies $(\tau, \gamma)$-\emph{growth condition} or $\mu$-\emph{local strong convexity} at $w^*$, i.e., $\forall w \in \cW$,
\begin{align*}
\min\left\{\frac{\mu}{2}\|w-w^*\|_2^2,\frac{\tau}{2}\|w-w^*\|_2^\gamma \right\} \leq L(w)-L(w^*),
\end{align*}
for constants $\mu,\tau >0,\mu < d_w$ and $0<\gamma<2$.
Also, $L(w)$ satisfies a $c$-\emph{local self-concordance} assumption at $w^*$, i.e., for all $w$ s.t. $\|w-w^*\|_{\nabla^2L(w^*)}\leq c$,
$$
(1-c)^2\cdot \nabla^2L(w^*) \preceq \nabla^2L(w)\preceq (1-c)^{-2} \cdot \nabla^2L(w^*).
$$
We also assume $c \leq 0.5$ for convenience. This is without loss of generality because if the condition holds for $c>0.5$, then the condition for $c\leq 0.5$ is automatically satisfied.
\end{assumption}

\begin{figure}[t]
	\centering
	\begin{minipage}{0.99\linewidth}\centering
		\includegraphics[width=\textwidth]{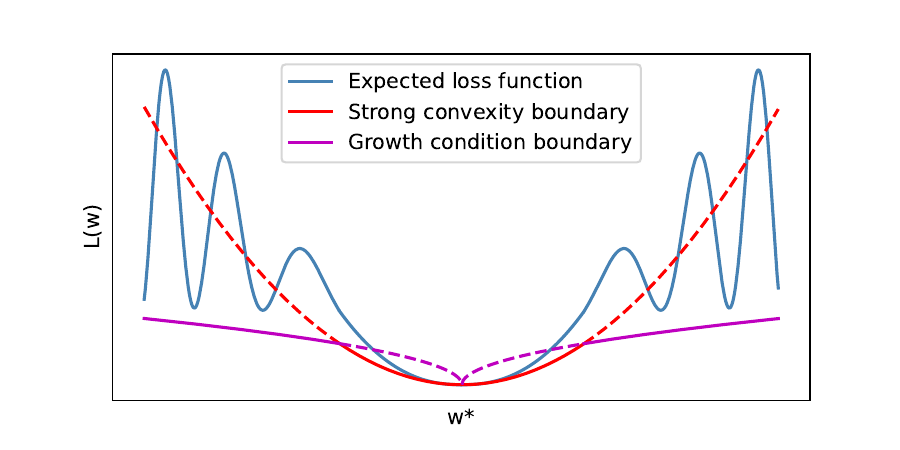}
	\end{minipage}
\caption{Example of a highly non-convex $L(w)$ satisfying Assumption \ref{ass:loss}. Solid lines denote the actual lower bound by taking $\min$ over strong convexity and growth condition. $L(w)$ is strongly convex near $w^*$ but can be highly non-convex away from $w^*$.
}\label{fig:loss}
\end{figure}

This assumption has three main components: (global) growth condition,  local strong convexity, and local self-concordance. 

The global growth condition says that $f_w$ with parameters far away from $w^*$ cannot approximate $f$ well over the distribution $\cU$. The local strong convexity assumption requires the local neighborhood near $w^*$ to have quadratic growth.  

These two conditions are strictly weaker than global strong convexity because it does not require convexity except in a local neighborhood near the global optimal $w^*$, i.e., $\{w|\|w - w^*\|_2 \leq (\tau/\mu)^\frac{1}{2-\gamma}\}$ and it does not limit the number of spurious local minima, as the global $\gamma$-growth condition only gives a mild lower bound as $w$ moves away from $w^*$. See Figure \ref{fig:loss} for an example. Our results work even if $\gamma$ is a small constant $<1$.

Self-concordance originates from a clean analysis of Newton's method \citep{nesterov1994interior}. See Example 4 of \citet{zhang2017improved} for a list of examples satisfying self-concordance. A localized version of self-concordance is needed in our problem for technical reasons, but again it is only required within a small ball of radius $c$ near $w^*$ for the expected loss under $\cU$.
Our results work even if $c$ vanishes at $O(T^{-1/4})$.

To avoid any confusion, the three assumptions we made above are only about the expected loss function w.r.t. uniform distribution $\cU$ as a function of $w$, rather than objective function $f_{w^*}(x)$. The problem to be optimized can still be arbitrarily complex in terms of $\cX$, e.g., high-dimensional and non-continuous functions. As an example, in Gaussian process-based Bayesian optimization approaches, $f_{w^*}(x)$ belongs to a reproducing kernel Hilbert space, but its loss function is globally convex in its ``infinite dimensional'' parameter $w$. Also, we no longer need this assumption in Phase II. 

\textbf{Additional notations.} For convenience, we define $\zeta >0$ such that $\|\nabla^2 L(w^*)\|_\mathrm{op} \leq  \zeta$.
The existence of a finite $\zeta$ is implied by Assumption~\ref{ass:objective} and it suffices to take $\zeta = 2C_g^2$ because $\nabla^2 L(w^*) = \E_{x\sim \cU}[ 2\nabla f_x(w^*) \nabla f_x(w^*)^\top]$.

\section{Main Results}\label{sec:main}
In Section \ref{sec:alg}, we state our Global Optimization with Upper Confidence Bound (GO-UCB) algorithm and explain key design points of it. Then in Section \ref{sec:regret_bound}, we prove that its cumulative regret bound is at the rate of $\tilde{O}(\sqrt{T})$.

\subsection{Algorithm}\label{sec:alg}
Our GO-UCB algorithm, shown in Algorithm \ref{alg:go_ucb}, has two phases. Phase I does uniform exploration in $n$ rounds and Phase II does optimistic exploration in $T$ rounds. In Step 1 of Phase I, $n$ is chosen to be large enough such that the objective function can be sufficiently explored. Step 2-3 are doing uniform sampling. In Step 5, we call regression oracle to estimate $\hat{w}_0$ given all observations in Phase I as in eq. \eqref{eq:regression_oracle}.
Adapted from \citet{nowak2007complexity}, we prove the convergence rate of $\|\hat{w}_0 - w^*\|_2$ is at the rate of $\tilde{O}(1/\sqrt{n})$. See Theorem \ref{thm:mle_guarantee} for details.

\begin{algorithm}[!b]
\caption{GO-UCB}
	\label{alg:go_ucb}
	{\bf Input:}
	Time horizon $T$, uniform exploration phase length $n$, uniform distribution $\cU$, regression oracle $\mathrm{Oracle}$, regularization weight $\lambda$,
 confidence sequence $\beta_t$ for $t=1,2,...,T$.\\
	{\bf Phase I} (Uniform exploration)
	\begin{algorithmic}[1]
	    \FOR{$j = 1,...,n$}
	    \STATE Sample $x_j \sim \cU(\cX)$.
	    \STATE Observe $y_j = f(x_j) + \eta_j$.
		\ENDFOR
		\STATE Estimate $\hat{w}_0 \leftarrow \mathrm{Oracle}(x_1, y_1, ..., x_n, y_n)$.
	\end{algorithmic}
	{\bf Phase II} (Optimistic exploration)
	\begin{algorithmic}[1]
	    \FOR{$t = 1,...,T$}
	    \STATE Update $\Sigma_t$ by eq. \eqref{eq:sigma_t} with the input $\lambda$.
	    \STATE Update $\hat{w}_t$ by eq. \eqref{eq:inner}  with the input $\lambda$.
	    \STATE Update $\mathrm{Ball}_t$ by eq. \eqref{eq:ball} with the input $\beta_t$.
	    \STATE Select $x_t=\argmax_{x \in \cX} \max_{w \in \mathrm{Ball}_t} f_x(w)$.
	    \STATE Observe $y_t = f(x_t) + \eta_t$.
		\ENDFOR
	\end{algorithmic}
	{\bf Output:}
        $\hat{x} \sim \cU (\{x_1, ..., x_T\})$.
\end{algorithm}

The key challenge of Phase II of GO-UCB is to design an acquisition function to select $x_t, \forall t \in [T]$. Since we are using parametric function to approximate the objective function, we heavily rely on a feasible parameter uncertainty region $\mathrm{Ball}_t, \forall t \in [T]$, which should always contain the true parameter $w^*$ throughout the process. The shape of $\mathrm{Ball}_t$ is measured by the covariance matrix $\Sigma_t$, defined as
\begin{align}
\Sigma_t = \lambda I + \sum_{i=0}^{t-1} \nabla f_{x_i}(\hat{w}_i) \nabla f_{x_i}(\hat{w}_i)^\top.\label{eq:sigma_t}
\end{align}
Note $i$ is indexing over both $x$ and $w$, which means that as time $t$ goes from $0$ to $T$, the update to $\Sigma_t$ is always rank one. It allows us to bound the change of $\Sigma_t$ from $t=0$ to $T$.

$\mathrm{Ball}_t$ is centered at $\hat{w}_t$, the newly estimated parameter at round $t$. In Step 2, we update the estimated $\hat{w}_t$ by solving the following optimization problem:
\begin{align}
\hat{w}_t &= \argmin_{w}\frac{\lambda}{2} \|w - \hat{w}_0\|_2^2\nonumber \\
&+ \half \sum_{i=0}^{t-1} ((w-\hat{w}_i)^\top \nabla f_{x_i}(\hat{w}_i) + f_{x_i}(\hat{w}_i) - y_i)^2.\label{eq:opt_inner}
\end{align}
The optimization problem is an online regularized least square problem involving gradients from all previous rounds, i.e., $\nabla f_{x_i}(\hat{w}_i), \forall i \in [T]$. The intuition behind it is that we use gradients to approximate the function since we are dealing with generic objective function. We set the regularization w.r.t. $\hat{w}_0$ rather than $0$ because from regression oracle we know how close is $\hat{w}_0$ to $w^*$. By setting the gradient of objective function in eq. \eqref{eq:opt_inner} to be $0$, the closed form solution of $\hat{w}_t$ is 
\begin{align}
\hat{w}_t &= \Sigma^{-1}_t \left(\sum_{i=0}^{t-1} \nabla f_{x_i}(\hat{w}_i) (\nabla f_{x_i}(\hat{w}_i)^\top \hat{w}_i +y_i - f_{x_i}(\hat{w}_i)) \right) \nonumber\\
&\qquad + \lambda \Sigma^{-1}_t \hat{w}_0.\label{eq:inner}
\end{align}

Now we move to our definition of $\mathrm{Ball}_t$, shown as
\begin{align}
\mathrm{Ball}_t = \{w: \|w-\hat{w}_t\|^2_{\Sigma_t} \leq \beta_t\},\label{eq:ball}
\end{align}
where $\beta_t$ is a pre-defined monotonically increasing sequence that we will specify later. 
Following the ``optimism in the face of uncertainty'' idea, our ball is centered at $\hat{w}_t$ with $\beta_t$ being the radius and $\Sigma_t$ measuring the shape. $\beta_t$ ensures that the true parameter $w^*$  is always contained in $\mathrm{Ball}_t$  w.h.p. 
In Section \ref{sec:ball}, we will show that it suffices to choose 
\begin{align}
\beta_t &= \tilde{O}\left(d_w \sigma^2 + \frac{d^3_w}{\mu^2} 
+ \frac{d^3_w t}{\mu^2 T}\right),\label{eq:beta_t}
\end{align}
where $\tilde{O}$ hides logarithmic terms in $t, T$ and $1/\delta$ (w.p. $1-\delta$). 

Then in Step 5 of Phase II, $x_t$ is selected by joint optimization over $x \in \cX$ and $w \in \mathrm{Ball}_t$. Finally, we collect all observations in $T$ rounds and output $\hat{x}$ by uniformly sampling over $\{x_1,...,x_T\}$.

\subsection{Regret Upper Bound}\label{sec:regret_bound}
Now we present the cumulative regret upper bound of GO-UCB algorithm.

\begin{theorem}[Cumulative regret of GO-UCB]\label{thm:cr}
Suppose Assumption \ref{ass:parameter_class}, \ref{ass:objective}, \& \ref{ass:loss} hold with parameters $F,C_g,C_h,\zeta,\mu,\gamma,\tau,c$. Assume
\begin{equation}\label{eq:n_lower_main}
T > C d_w^2 F^4\iota^2 \cdot \max \left\{   \frac{\mu^{\gamma/(2-\gamma)}}{\tau^{2/(2-\gamma)}}, \frac{\zeta}{\mu c^2} \right\}^2,  
\end{equation}
where $C$ is a universal constant and $\iota$ is a logarithmic term depending on $n, C_h, 2/\delta$ (both of them from Theorem~\ref{thm:mle_guarantee}).  Then Algorithm~\ref{alg:go_ucb} with parameters $n = \sqrt{T}$, $\lambda = C_\lambda \sqrt{T}$ (for a $C_\lambda$ logarithmically dependent to $T$ and polynomial in all other parameters) and $\beta_{1:T}$ as in eq. \eqref{eq:beta_t}
obeys that with probability  at least $1-\delta$,
\begin{align*}
R_{\sqrt{T}+T} &= \tilde{O}\left(\sqrt{T} F + \sqrt{T\beta_T d_w  + \frac{T\beta^2_T}{\lambda^2}}\right)\\
& = \tilde{O}\left(\frac{d^2_w \sqrt{T}}{\mu}\right).
\end{align*}
\end{theorem}

Let us highlight a few interesting aspects of the result.
\begin{remark}
Without Gaussian process assumption, we propose the first algorithm to solve global optimization problem with $\tilde{O}(\sqrt{T})$ cumulative regret, which is \emph{dimension-free} in terms of its input domain $\cX$. GO-UCB is a no-regret algorithm since $\lim_{T \rightarrow \infty} R_T/T =0$, and the output $\hat{x}$ satisfies that 
$f^* - \E[ f(\hat{x})] \leq \tilde{O}(1/\sqrt{T})$, which is also knowns as expected simple regret upper bound. The dependence in $T$ is optimal up to logarithmic factors, as it matches the lower bound for linear bandits \citep[Theorem 3]{dani2008stochastic}. 
\end{remark}
\begin{remark}[Choice of $\lambda$]
One important deviation from the classical linear bandit analysis is that we require a regularization that centers around $\hat{w}_0$ and the regularization weight $\lambda$ to be $C_\lambda \sqrt{T}$, comparing to $\lambda = O(1)$ in the linear case. The choice is to ensure that $\hat{w}_t$ stays within the local neighborhood of $\hat{w}_0$, and to delicately balance different terms that appear in the regret analysis to ensure that the overall regret bound is $\tilde{O}(\sqrt{T})$.
\end{remark}
\begin{remark}[Choice of $n$]
We choose $n =\sqrt{T}$, therefore, it puts sample complexity requirement on $T$ shown in eq. \eqref{eq:n_lower_main}. The choice of $n$ plays two roles here. First, it guarantees that the regression result $\hat{w}_0$ lies in the neighboring region of $w^*$ of the loss function $L(w)$ with high probability. The neighboring region of $w^*$ has nice properties, e.g., local strong convexity, which allow us to build the upper bound of $\ell_2$-distance between $\hat{w}_0$ and $w^*$. Second, in Phase I, we are doing uniform sampling over the function so the cumulative regret in Phase I is bounded by $2Fn=2F\sqrt{T}$ which is at the same $\tilde{O}(\sqrt{T})$ rate as that in Phase II.
\end{remark}

\section{Proof Overview}\label{sec:proof}
In this section, we give a proof sketch of all theoretical results. A key insight of our analysis is that there is more mileage that seminal techniques developed by \citet{abbasi2011improved} for analyzing linearly parameterized bandits problems in analyzing non-linear bandits, though we need to localize to a nearly optimal region and carefully handle the non-linear components via more aggressive regularization. Other assumptions that give rise to a similarly good initialization may work too and our new proof can be of independent interest in analyzing other extensions of LinUCB, e.g., to contextual bandits,  reinforcement learning and other problems.

In detail, first we prove the estimation error bound of $\hat{w}_0$ for Phase I of GO-UCB algorithm, then 
prove the feasibility of $\mathrm{Ball}_t$. Finally by putting everything together we prove the cumulative regret bound of GO-UCB algorithm. Due to page limit, we list all auxiliary lemmas in Appendix \ref{sec:auxiliary} and show complete proofs in Appendix \ref{sec:miss}.

\subsection{Regression Oracle Guarantee}\label{sec:mle}
The goal of Phase I of GO-UCB is to sufficiently explore the unknown objective function with $n$ uniform queries and obtain an estimated parameter $\hat{w}_0$. By assuming access to a regression oracle, we prove the convergence bound of $\hat{w}_0$ w.r.t. $w^*$, i.e., $\|\hat{w}_0-w^*\|^2_2$. To get started, we need the following regression oracle lemma.

\begin{lemma}[Adapted from \citet{nowak2007complexity})]\label{lem:mle_oracle}
Suppose Assumption \ref{ass:parameter_class} \& \ref{ass:objective} hold. There is an absolute constant $C'$, such that after round $n$ in Phase I of Algorithm \ref{alg:go_ucb}, with probability $>1 - \delta/2$, regression oracle estimated $\hat{w}_0$ satisfies
\begin{align*}
\E_{x \sim \cU} [(f_x(\hat{w}_0) - f_x(w^*))^2] \leq \frac{C' d_w F^2 \iota}{n},
\end{align*}
where $\iota$ is the logarithmic term depending on $n, C_h, 2/\delta$.
\end{lemma}

\citet{nowak2007complexity} proves that expected square error of Empirical Risk Minimization (ERM) estimator can be bounded at the rate of $\tilde{O}(1/n)$ with high probability, rather than $\tilde{O}(1/\sqrt{n})$ rate achieved by Chernoff/Hoeffding bounds. It works with realizable and misspecified settings. Proof of Lemma \ref{lem:mle_oracle} includes simplifying it with regression oracle, Assumption \ref{ass:parameter_class}, and $\varepsilon$-covering number argument over parameter class. Basically Lemma \ref{lem:mle_oracle} says that expected square error of $f_x(\hat{w}_0)$ converges to $f_x(w^*)$ at the rate of $\tilde{O}(1/n)$ with high probability. Based on it, we prove the following regression oracle guarantee.

\begin{theorem}[Regression oracle guarantee]\label{thm:mle_guarantee}
Suppose Assumption \ref{ass:parameter_class}, \ref{ass:objective}, \& \ref{ass:loss} hold. There is an absolute constant $C$ such that after round $n$ in Phase I of Algorithm \ref{alg:go_ucb} where $n$ satisfies
$n \geq C d_w F^2\iota \cdot \max \left\{   \frac{\mu^{\gamma/(2-\gamma)}}{\tau^{2/(2-\gamma)}}, \frac{\zeta}{\mu c^2} \right\}$,
with probability $> 1-\delta/2$, regression oracle estimated $\hat{w}_0$ satisfies
\begin{align*}
\|\hat{w}_0 - w^*\|^2_2 \leq \frac{C d_w F^2 \iota}{\mu n},
\end{align*}
where $\iota$ is the logarithmic term depending on $n, C_h, 2/\delta$.
\end{theorem}

Compared with Lemma \ref{lem:mle_oracle}, there is an extra sample complexity requirement on $n$ because we need $n$ to be sufficiently large such that the function can be sufficiently explored and more importantly $\hat{w}_0$ falls into the neighboring region (strongly convex region) of $w^*$. See Figure \ref{fig:loss} for illustration. It is also the reason why strong convexity parameter $\mu$ appears in the denominator of the upper bound. 

\subsection{Feasibility of $\mathrm{Ball}_t$}\label{sec:ball}
The following lemma is the key part of algorithm design of GO-UCB. It says that our definition of $\mathrm{Ball}_t$ is appropriate, i.e., throughout all rounds in Phase II, $w^*$ is contained in $\mathrm{Ball}_t$ with high probability.

\begin{lemma}[Feasibility of $\mathrm{Ball}_t$]\label{lem:feasible_ball}
Set $\Sigma_t, \hat{w}_t$ as in eq. \eqref{eq:sigma_t}, \eqref{eq:inner}. Set $\beta_t$ as
\begin{align}
\beta_t = \tilde{O} \left(d_w \sigma^2 + \frac{d^3_w}{\mu^2} + \frac{d^3_w t }{\mu^2 T}\right).\label{eq:beta_2}
\end{align}
Suppose Assumption \ref{ass:parameter_class}, \ref{ass:objective}, \& \ref{ass:loss} hold and choose $n= \sqrt{T}, \lambda = C_\lambda \sqrt{T}$. Then $\forall t \in [T]$ in Phase II of Algorithm \ref{alg:go_ucb}, w.p. $>1-\delta$,
\begin{align*}
\|\hat{w}_t - w^*\|^2_{\Sigma_t} &\leq \beta_t.
\end{align*}
\end{lemma}
For reader's easy reference, we write our choice of $\beta_t$ again in eq. \eqref{eq:beta_2}. Note this lemma requires careful choices of $\lambda$ and $n$ because $\beta_t$ appears later in the cumulative regret bound and $\beta_t$ is required to be at the rate of $\tilde{O}(1)$. The proof has three steps. First we obtain the closed form solution of $\hat{w}_t$ as in eq. \eqref{eq:inner}. Next we use induction to prove that $\forall t \in [T], \|\hat{w}_t - w^*\|^2_2 \leq \tilde{O}(\tilde{C}/n)$ for some universal constant $\tilde{C}$. Finally we prove $\|\hat{w}_t - w^*\|^2_{\Sigma_t} \leq \beta_t$.

\subsection{Regret Analysis}\label{sec:reg}
To prove cumulative regrets bound of GO-UCB algorithm, we need following two lemmas of instantaneous regrets in Phase II of GO-UCB.

\begin{lemma}[Instantaneous regret bound]\label{lem:instant_regret}
Set $\Sigma_t, \hat{w}_t, \beta_t$ as in eq. \eqref{eq:sigma_t}, \eqref{eq:inner}, \& \eqref{eq:beta_t} and suppose Assumption \ref{ass:parameter_class}, \ref{ass:objective}, \& \ref{ass:loss} hold, then with probability $> 1- \delta$,
$w^*$ is contained in $\mathrm{Ball}_t$.
Define $u_t = \|\nabla f_{x_t}(\hat{w}_t)\|_{\Sigma^{-1}_t}$, then $\forall t \in [T]$ in Phase II of Algorithm \ref{alg:go_ucb},
\begin{align*}
r_t \leq 2\sqrt{\beta_t}u_t + \frac{2\beta_t C_h}{\lambda}.
\end{align*}
\end{lemma}
The first term of the upper bound is pretty standard, seen also in LinUCB \citep{abbasi2011improved} and GP-UCB \citep{srinivas2010gaussian}. After we apply first order gradient approximation of the objective function, the second term is the upper bound of the high order residual term, which introduces extra challenge to derive the upper bound. 

Technically, proof of Lemma \ref{lem:instant_regret} requires $w^*$ is contained in our parameter uncertainty ball $\mathrm{Ball}_t$ with high probability throughout Phase II of GO-UCB, which has been proven in Lemma \ref{lem:feasible_ball}. Later, the proof utilizes Taylor's theorem and uses the convexity of $\mathrm{Ball}_t$ twice. See Appendix \ref{sec:regret}. The next lemma is an extension of Lemma \ref{lem:instant_regret}, where the proof uses monotonically increasing property of $\beta_t$ in $t$. 

\begin{lemma}[Summation of squared instantaneous regret bound]\label{lem:sos_instant_regret}
Set $\Sigma_t, \hat{w}_t, \beta_t$ as in eq. \eqref{eq:sigma_t}, \eqref{eq:inner}, \& \eqref{eq:beta_t} and suppose Assumption \ref{ass:parameter_class}, \ref{ass:objective}, \& \ref{ass:loss} hold, then with probability $> 1- \delta$,
$w^*$ is contained in $\mathrm{Ball}_t$ and $\forall t \in [T]$ in Phase II of Algorithm \ref{alg:go_ucb},
\begin{align*}
\sum_{t=1}^T r^2_t \leq 16\beta_T d_w \log \left(1 + \frac{TC_g^2}{d_w \lambda}\right) + \frac{8\beta^2_T C^2_h T}{\lambda^2}.
\end{align*}
\end{lemma}

Proof of Theorem \ref{thm:cr} follows by putting everything together via Cauchy-Shwartz inequality
$\sum_{t=1}^T r_t \leq \sqrt{T\sum_{t=1}^T r_t^2}$.

\begin{figure*}[t]
	\centering
	\begin{minipage}{0.32\linewidth}\centering
		\includegraphics[width=\textwidth]{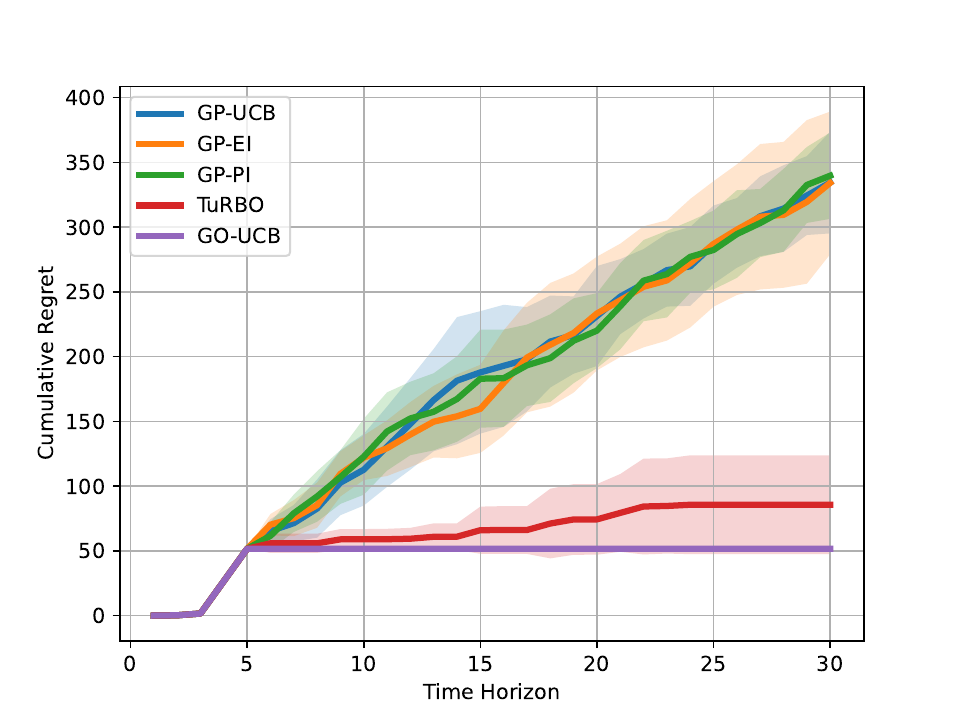}
		(a) $f_1$ (realizable)
	\end{minipage}
	\begin{minipage}{0.32\linewidth}\centering
		\includegraphics[width=\textwidth]{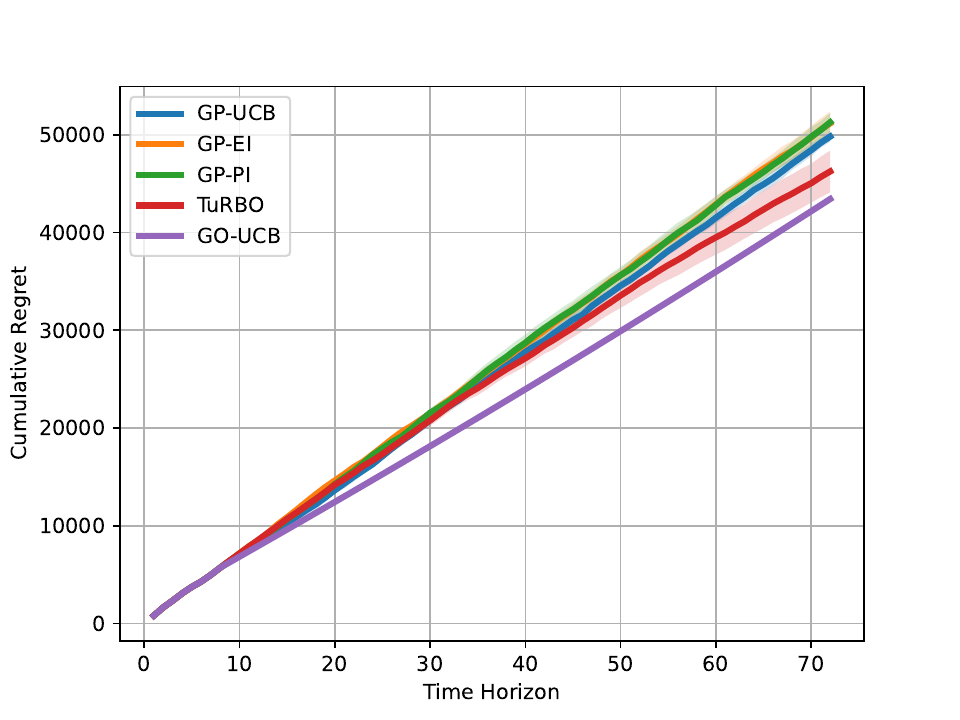}
		(b) $f_2$ (misspecified)
	\end{minipage}
	\begin{minipage}{0.32\linewidth}\centering
		\includegraphics[width=\textwidth]{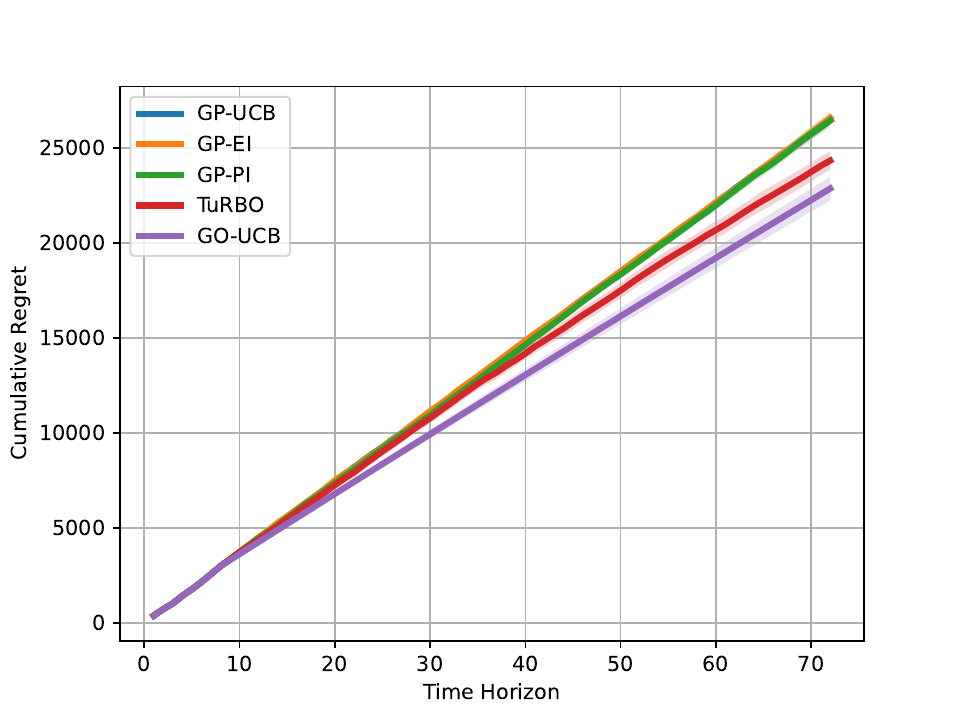}
		(c) $f_3$ (misspecified)
	\end{minipage}
\caption{Cumulative regrets (the lower the better) of all algorithms on $20$-dimensional $f_1, f_2, f_3$ synthetic functions.
}\label{fig:simulation}
\end{figure*}

\begin{figure*}[!htbp]
	\centering
	\begin{minipage}{0.32\linewidth}\centering
		\includegraphics[width=\textwidth]{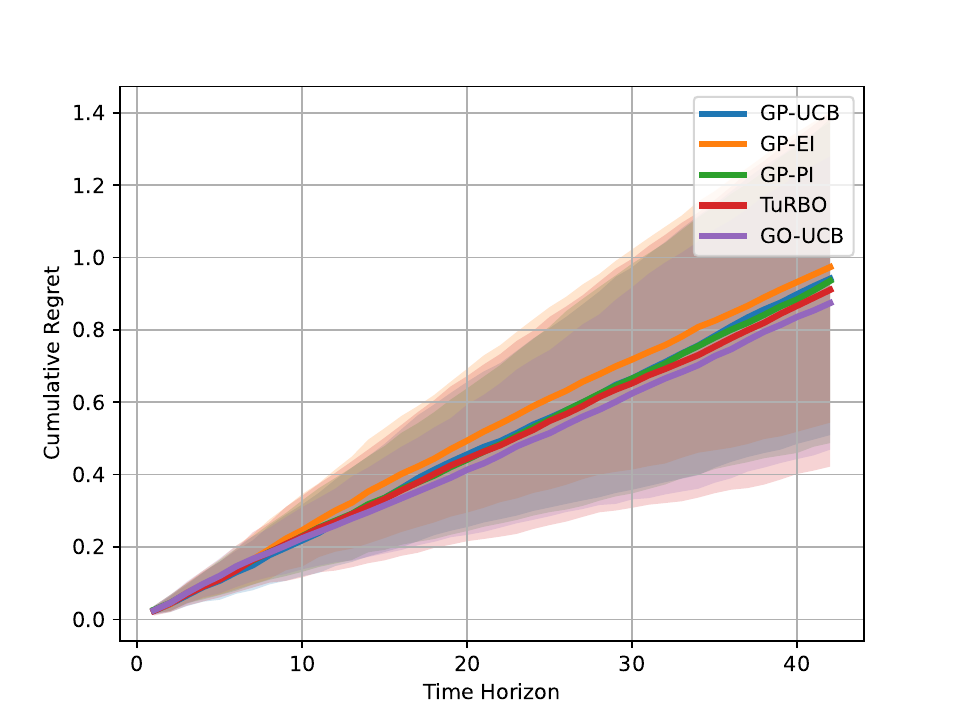}
		(a) Random forest ($d_x = 7$)
	\end{minipage}
	\begin{minipage}{0.32\linewidth}\centering
		\includegraphics[width=\textwidth]{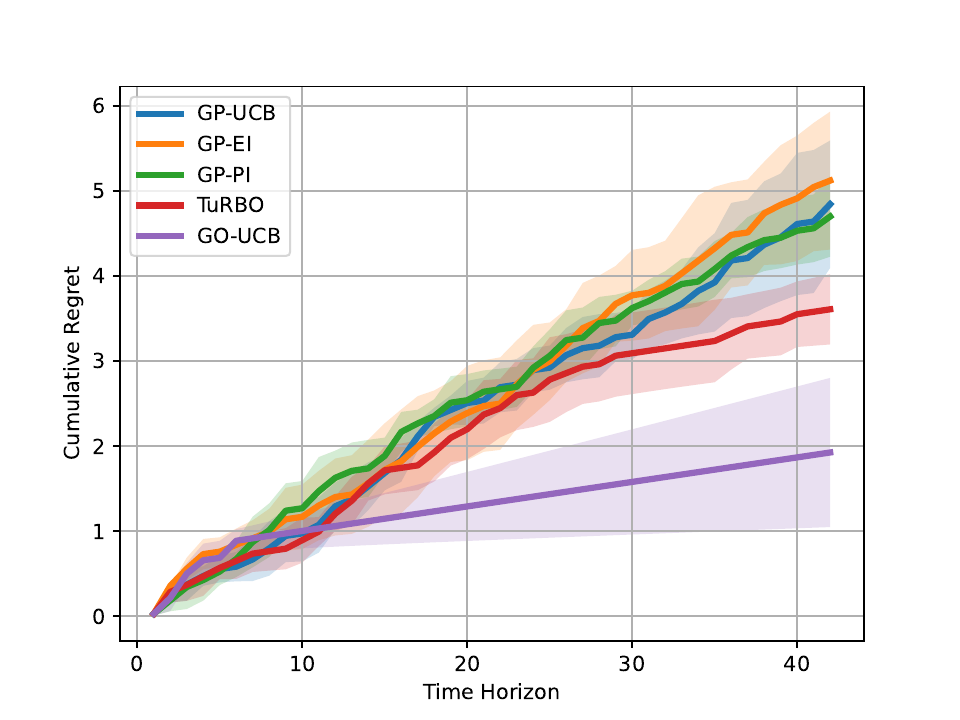}
		(b) Multi-layer perceptron ($d_x = 8$)
	\end{minipage}
	\begin{minipage}{0.32\linewidth}\centering
		\includegraphics[width=\textwidth]{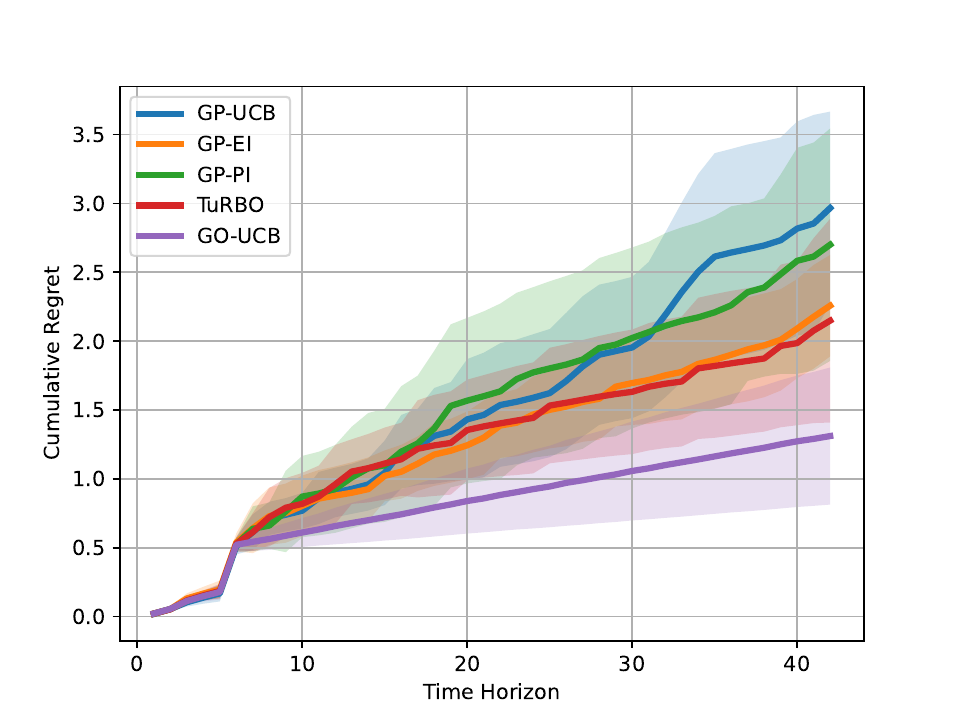}
		(c) Gradient boosting ($d_x = 11$)
	\end{minipage}
\caption{Cumulative regrets (the lower the better) of all algorithms in real-world hyperparameter tuning task on Breast-cancer dataset.
}\label{fig:real}
\end{figure*}

\section{Experiments}\label{sec:experiment}
We compare our GO-UCB algorithm with four Bayesian Optimization (BO) algorithms: GP-EI \citep{jones1998efficient}, GP-PI \citep{kushner1964new}, GP-UCB \citep{srinivas2010gaussian}, and Trust Region BO (TuRBO) \citep{eriksson2019scalable}, where the first three are classical methods and TuRBO is a more advanced algorithm designed for high-dimensional cases. 

To run GO-UCB, we choose our parametric function model $\hat{f}$ to be a two linear layer neural network with \texttt{sigmoid} function being the activation function:
\begin{align*}
\hat{f}(x) = \texttt{linear2}(\texttt{sigmoid}(\texttt{linear1}(x))),
\end{align*}
where $w_1, b_1$ denote the weight and bias of $\texttt{linear1}$ layer and $w_2, b_2$ denote those of $\texttt{linear2}$ layer.  Specifically, we set $w_1 \in \mathbb{R}^{25\times d_x}, b_1 \in \mathbb{R}^{25}, w_2 \in \mathbb{R}^{25}, b_2 \in \mathbb{R}$, meaning the dimension of activation function is $25$. All implementations are based on BoTorch framework \citep{balandat2020botorch} and sklearn package \citep{head2021skopt} with default parameter settings. To help readers reproduce our results, implementation details are shown in Appendix \ref{sec:imp_goucb}.

\subsection{Synthetic Experiments}\label{sec:syn}
First, we test all algorithms on three high-dimensional synthetic functions defined on $[-5, 5]^{d_x}$ where $d_x = 20$, including both realizable and misspecified cases. The first test function $f_1$ is created by setting all elements in $w_1, b_1, w_2, b_2$ in $\hat{f}$ to be $1$, so $f_1$ is a realizable function given $\hat{f}$. The second and third test functions $f_2, f_3$ are Styblinski-Tang function and Rastrigin function, defined as:
\begin{align*}
f_2 &= -\half \sum_{i=1}^{20} x_i^4 - 16 x^2_i + 5 x_i,\\
f_3 &= -200 + \sum_{i=1}^{20} 10 \cos(2 \pi x_i) - x^2_i,
\end{align*}
where $x_i$ denotes the $i$-th element in its $20$ dimensions, so $f_2, f_3$ are misspecified functions given $\hat{f}$.
We set $n=5, T=25$ for $f_1$ and $n=8,T=64$ for $f_2, f_3$. To reduce the effect of randomness in all algorithms, we repeat the whole optimization process for $5$ times for all algorithms and report mean and error bar of cumulative regrets. The error bar is measured by Wald's test with $95\%$ confidence, i.e., $1.96 \nu/\sqrt{5}$ where $\nu$ is standard deviation of cumulative regrets and $5$ is the number of repetitions.

From Figure \ref{fig:simulation}, we learn that in all tasks our GO-UCB algorithm performs better than all other four BO approaches. Among BO approches, TuRBO performs the best since it is specifically designed for high-dimensional tasks. In Figure \ref{fig:simulation}(a), mean of cumulative regrets of GO-UCB and TuRBO stays the same when $t \geq 22$, which means that both of them have found the global optima, but GO-UCB algorithm is able to find the optimal point shortly after Phase I and enjoys the least error bar. It is well expected since $f_1$ is a realizable function for $\hat{f}$. Unfortunately, GP-UCB, GP-EI, and GP-PI incur almost linear regrets, showing the bad performances of classical BO algorithms in high-dimensional cases.

In Figure \ref{fig:simulation}(b) and \ref{fig:simulation}(c), all methods are suffering from linear regrets because $f_2, f_3$ are misspecified functions. The gap between GO-UCB and other methods is smaller in Figure \ref{fig:simulation}(c) than in \ref{fig:simulation}(b) because optimizing $f_3$ is more challenging than $f_2$ since $f_3$ has more local optimal points.

\subsection{Real-World Experiments}\label{sec:real}
To illustrate the GO-UCB algorithm works in real-world tasks, we do hyperparameter tuning experiments on three tasks using three classifiers. Three UCI datasets \citep{Dua:2019} are Breat-cancer, Australian, and Diabetes, and three classifiers are random forest, multi-layer perceptron, and gradient boosting where each of them has $7,8,11$ hyperparameters. For each classifier on each dataset, the function mapping from hyperparameters to classification accuracy is the black-box function that we are maximizing, so the input space dimension $d_x=7,8,11$ for each classifier. We use cumulative regret to evaluate hyperparameter tuning performances, however, best accuracy $f^*$ is unknown ahead of time so we set it to be the best empirical accuracy of each task. To reduce the effect of randomness, we divide each dataset into 5 folds and every time use 4 folds for training and remaining 1 fold for testing. We report mean and error bar of cumulative regrets where error bar is measured by Wald's test, the same as synthetic experiments.

Figure \ref{fig:real} shows results on Breast-cancer dataset. In Figure \ref{fig:real}(b)(c) GO-UCB performs statistically much better that all other BO algorithms since there is almost no error bar gap between TuRBO and GO-UCB. It shows that GO-UCB can be deployed in real-world applications to replace BO methods. Also, in Figure \ref{fig:real}(b) performance of GO-UCB Phase I is not good but GO-UCB can still perform better than others in Phase II, which shows the effectiveness of Phase II of GO-UCB. In Figure \ref{fig:real}(a) all algorithms have similar performances. In Figure \ref{fig:real}(b), TuRBO performs similarly as GP-UCB, GP-EI, and GP-PI when $t \leq 23$, but after $t=23$ it performs better and shows a curved regret line by finding optimal points. Due to page limit, results on Australian and Diabetes datasets are shown in Appendix \ref{sec:real_detail} where similar algorithm performances can be seen.

Note in experiments, we choose parametric model $\hat{f}$ to be a two linear layer neural network. In more real-world experiments, one can choose the model $\hat{f}$ in GO-UCB to be simpler functions or much more complex functions, e.g., deep neural networks, depending on task requirements.

\section{Conclusion}\label{sec:conclusion}
Global non-convex optimization is an important problem that widely exists in many real-world applications, e.g., deep learning hyper-parameter tuning and new material design. However, solving this optimization problem in general is NP-hard. Existing work relies on Gaussian process assumption, e.g., Bayesian optimization, or other non-parametric family which suffers from the curse of dimensionality. 

We propose the first algorithm to solve such global optimization with parametric function approximation, which shows a new way of global optimization. GO-UCB first uniformly explores the function and collects a set of observation points and then uses the optimistic exploration to actively select points. At the core of GO-UCB is a carefully designed uncertainty set over parameters based on gradients that allows optimistic exploration. Under realizable parameter class assumption and a few mild geometric conditions, our theoretical analysis shows that cumulative regret of GO-UCB is at the rate of $\tilde{O}(\sqrt{T})$, which is dimension-free in terms of function domain $\cX$. 
Our high-dimensional synthetic test shows that GO-UCB works better than BO methods even in misspecified setting. Moreover, GO-UCB performs better than BO algorithms in real-world hyperparameter tuning tasks, which may be of independent interest. 
 
There is $\mu$, the strongly convexity parameter, in the denominator of upper bound in Theorem \ref{thm:cr}. $\mu$ can be small in practice, thus the upper bound can be large. Developing the cumulative regret bound containing a term depending on $\mu$ but being independent to $T$ remains a future problem.

\subsection*{Acknowledgments}
The work is partially supported by NSF Awards \#1934641 and \#2134214. We thank Ming Yin and Dan Qiao for helpful discussion and careful proofreading of an early version of the manuscript, as well as Andrew G. Wilson for pointing out that Bayesian optimization methods do not necessarily use Gaussian processes as surrogate models. Finally, we thank ICML reviewers and the area chair for their valuable input that led to improvements to the paper. 

\bibliographystyle{icml2023}
\bibliography{bib}

\newpage
\onecolumn
\appendix

\section{Notation Table}\label{sec:table}
\begin{table}[!htbp]
\centering
\caption{Symbols and notations.}\label{tab:notations}
\begin{tabular}{ccl}
\noalign{\smallskip} \hline
\textbf{Symbol}  & \textbf{Definition} & \textbf{Description} \\ \hline
$\|A\|_\mathrm{op}$ & & operator norm\\ \hline
$\mathrm{Ball}_t$ &  eq. \eqref{eq:ball} & parameter uncertainty region at round $t$\\ \hline
$\beta_t$ & eq. \eqref{eq:beta_t}  & parameter uncertainty region radius at round $t$\\ \hline
$\mu$ & & local strong convexity parameter \\ \hline
$c$ & & local self-concordance parameter \\ \hline
$C, \zeta$ & & constants\\ \hline
$d_x$       &      &  domain dimension     \\ \hline
$d_w$       &      &  parameter dimension     \\ \hline
$\delta$       &      &  failure probability \\ \hline
$\varepsilon$       &      &  covering number discretization distance    \\ \hline
$\eta$       & $\sigma$-sub-Gaussian     &  observation noise     \\ \hline
$f_w(x)$       &      &  objective function at $x$ parameterized by $w$    \\ \hline
$f_x(w)$       &      &  objective function at $w$ parameterized by $x$    \\ \hline
$\nabla f_x(w)$       &      &  1st order derivative w.r.t. $w$ parameterized by $x$    \\ \hline
$\nabla^2 f_x(w)$       &      &  2nd order derivative w.r.t. $w$ parameterized by $x$    \\ \hline
$F$ &           &  function range constant bound   \\ \hline 
$\gamma, \tau$ &           &  growth condition parameters     \\ \hline
$\iota, \iota', \iota{''}$ &           &  logarithmic terms     \\ \hline
$L(w)$ & $\E[(f_x(w) - f_x(w^*))^2]$          &  expected loss function     \\ \hline 
$\lambda$ & & regularization parameter \\ \hline
$n$             &      & time horizon in Phase I \\ \hline
$[n]$             &  $\{1,2,...,n\}$    & integer set of size $n$\\ \hline
$\mathrm{Oracle}$             &      & regression oracle \\ \hline
$r_t$ & $f_{w^*}(x^*) - f_{w^*}(x_t)$ & instantaneous regret at round $t$ \\ \hline
$R_T$ & $\sum_{t=1}^T r_t$ & cumulative regret after round $T$\\ \hline
$\Sigma_t$ & eq. \eqref{eq:sigma_t}           & covariance matrix at round $t$ \\ \hline
$T$ &           & time horizon in Phase II    \\ \hline
$\cU$ & & uniform distribution \\ \hline
$w$ &  $w \in \cW$         & function parameter    \\ \hline 
$w^*$ &  $w^* \in \cW$         & true parameter    \\ \hline 
$\hat{w}_0$ &           & oracle-estimated parameter after Phase I    \\ \hline 
$\hat{w}_t$ & eq. \eqref{eq:inner}          & updated parameter at round $t$\\ \hline 
$\cW$ & $\cW \subseteq [0,1]^{d_w}$ & parameter space \\ \hline
$x$ & $x \in \cX$ & data point \\\hline
$x^*$ & & optimal data point \\\hline
$\|x\|_\infty$ &$(\sum_{i=1}^d |x_i|^p)^{1/p}$ & $\ell_p$ norm \\\hline
$\|x\|_p$ &$\max_{i \in [d]} |x_i|$ & $\ell_\infty$ norm \\\hline
$\|x\|_A$ &$\sqrt{x^\top A x}$ & distance defined by square matrix $A$ \\\hline
$\cX$   & $\cX \subseteq \mathbb{R}^{d_x}$        &  function domain     \\  \hline
$\cY$   & $\cY = [-F, F]$        &  function range     \\  \hline
\end{tabular}
\end{table}

\section{Auxiliary Technical Lemmas}\label{sec:auxiliary}
In this section, we list auxiliary lemmas that are used in proofs.

\begin{lemma}[Adapted from eq. (5) (6) of  \citet{nowak2007complexity}]\label{lem:regression}
Given a dataset $\{x_i,y_i\}_{j=1}^n$ where $y_j$ is generated from eq. \eqref{eq:y} and $f_0$ is the underlying true function. Let $\hat{f}$ be an ERM estimator taking values in $\cF$ where $\cF$ is a finite set and 
$\cF \subset \{f: [0,1]^d \rightarrow [-F,F]\}$ for some $F \geq 1$. Then with probability $> 1- \delta$, $\hat{f}$ satisfies that
\begin{align*}
\E[(\hat{f} - f_0)^2] \leq \left(\frac{1+\alpha}{1-\alpha}\right) \left( \inf_{f\in \cF } \E[(f - f_0)^2] + \frac{F^2 \log(|\cF|) \log(2)}{n\alpha}\right) + \frac{2\log(2/\delta)}{n\alpha},
\end{align*}
for all $\alpha \in (0, 1]$.
\end{lemma}

\begin{lemma}[Sherman-Morrison lemma \citep{sherman1950adjustment}]\label{lem:sherman}
Let $A$ denote a matrix and $b,c$ denote two vectors. Then
\begin{align*}
(A + bc^\top)^{-1} = A^{-1} - \frac{A^{-1}bc^\top A^{-1}}{1+ c^\top A^{-1} b}.
\end{align*}
\end{lemma}

\begin{lemma}[Self-normalized bound for vector-valued martingales \citep{abbasi2011improved,agarwal2021rl}]\label{lem:self_norm}
Let $\{\eta_i\}_{i=1}^\infty$ be a real-valued stochastic process with corresponding ﬁltration $\{\cF_i\}_{i=1}^\infty$ such that $\eta_i$ is $\cF_i$ measurable, $\E[\eta_i | \cF_{i-1} ] = 0$, and $\eta_i$ is conditionally $\sigma$-sub-Gaussian with $\sigma \in \mathbb{R}^+$. Let $\{X_i\}_{i=1}^\infty$ be a stochastic process with $X_i \in \cH$ (some Hilbert space) and $X_i$ being $F_t$ measurable. Assume that a linear operator $\Sigma:\cH \rightarrow \cH$ is positive deﬁnite, i.e., $x^\top \Sigma x > 0$ for any $x \in \cH$. For any $t$, deﬁne the linear operator $\Sigma_t = \Sigma_0 + \sum_{i=1}^t X_i X_i^\top$ (here $xx^\top$ denotes outer-product in $\cH$). With probability at least $1-\delta$, we have for all $t\geq 1$:
\begin{align*}
\left\|\sum_{i=1}^t X_i \eta_i \right\|^2_{\Sigma_t^{-1}} \leq \sigma^2 \log \left(\frac{\det(\Sigma_t) \det(\Sigma_0)^{-1}}{\delta^2} \right).
\end{align*}
\end{lemma}

\section{Missing Proofs}\label{sec:miss}
In this section, we show complete proofs of all technical results in the main paper. For reader's easy reference, we define $\iota$ as a logarithmic term depending on $n, C_h, 2/\delta$ (w.p. $>1-\delta/2$), $\iota'$ as a logarithmic term depending on $t, d_w, C_g, 1/\lambda, 2/\delta$ (w.p. $>1-\delta/2$), and $\iota{''}$ as a logarithmic term depending on $t, d_w, C_g, 1/\lambda$.

\subsection{Regression Oracle Guarantee}

\begin{lemma}[Restatement of Lemma \ref{lem:mle_oracle}]
Suppose Assumption \ref{ass:parameter_class} \& \ref{ass:objective} hold. There is an absolute constant $C'$, such that after round $n$ in Phase I of Algorithm \ref{alg:go_ucb}, with probability $>1 - \delta/2$, regression oracle estimated $\hat{w}_0$ satisfies
\begin{align*}
\E_{x \sim \cU} [(f_x(\hat{w}_0) - f_x(w^*))^2] \leq \frac{C' d_w F^2 \iota}{n},
\end{align*}
where $\iota$ is the logarithmic term depending on $n, C_h, 2/\delta$.
\end{lemma}
\begin{proof}
The regression oracle lemma establishes on Lemma \ref{lem:regression} which works only for finite function class. In order to work with our continuous parameter class $\cW$, we need $\varepsilon$-covering number argument. 

First, let $\tilde{w}, \widetilde{\cW}$ denote the ERM parameter and finite parameter class after applying covering number argument on $\cW$. By Lemma \ref{lem:regression}, we find that with probability $>1-\delta/2$,
\begin{align*}
\E_{x \sim \cU}[(f_x(\tilde{w}) - f_x(w^*))^2] &\leq \left(\frac{1+\alpha}{1-\alpha}\right) \left( \inf_{w \in \widetilde{\cW} \cup \{w^*\}} \E_{x\sim \cU}[(f_x(w) - f_x(w^*))^2] + \frac{F^2  \log(|\widetilde{\cW}|) \log(2)}{n\alpha}\right) + \frac{2\log(4/\delta)}{n\alpha}\\
&\leq \left(\frac{1+\alpha}{1-\alpha}\right) \left( \frac{F^2 \log(|\widetilde{\cW}|) \log(2)}{n\alpha}\right) + \frac{2\log(4/\delta)}{n\alpha},
\end{align*}
where the second inequality is by realizable assumption (Assumption \ref{ass:parameter_class}).
Our parameter class $\cW \subseteq [0, 1]^{d_w}$, so $\log(|\widetilde{\cW}|) = \log(1/\varepsilon^{d_w})= d_w\log(1/\varepsilon)$ and the new upper bound is that with probability $> 1-\delta/2$,
\begin{align*}
\E_{x\sim \cU}[(f_x(\tilde{w}) - f_x(w^*))^2] \leq C^{''}\left(\frac{d_w F^2 \log(1/\varepsilon)}{n} + \frac{\log(2/\delta)}{n}\right),
\end{align*}
where $C^{''}$ is a universal constant obtained by choosing $\alpha=1/2$. Note $\tilde{w}$ is the ERM parameter in $\widetilde{\cW}$ after discretization, not our target parameter $\hat{w}_0 \in \cW$. By $(a+b)^2\leq 2a^2 + 2b^2$,
\begin{align}
\E_{x\sim \cU}[(f_x(\hat{w}_0) - f_x(w^*))^2] &\leq 2\E_{x\sim \cU}[(f_x(\Hat{w}_0) - f_x(\tilde{w}))^2] + 2\E_{x\sim \cU}[(f_x(\tilde{w}) - f_x(w^*))^2] \nonumber \\
&\leq 2\varepsilon^2 C_h^2 + 2C^{''}\left(\frac{d_w F^2 \log(1/\varepsilon)}{n} + \frac{\log(2/\delta)}{n}\right)\label{eq:interm1}
\end{align}
where the second line applies discretization error $\varepsilon$ and Assumption \ref{ass:objective}.
By choosing $\varepsilon = 1/\sqrt{n C_h^2}$, 
we get
$$
\eqref{eq:interm1} =\frac{2}{n} + \frac{C^{''}d_w F^2\log(n C_h^2)}{n} + \frac{2C^{''}\log(2/\delta)}{n} \leq C' \frac{d_w F^2\log(n C_h^2) +\log(2/\delta)}{n}$$
where we can take $C' = 2C^{''}$ (assuming $2<C^{''}d_w F^2\log(n C_h^2)$). The proof completes by defining $\iota$ as the logarithmic term depending on $n, C_h, 2/\delta$.
\end{proof}

\begin{theorem}[Restatement of Theorem \ref{thm:mle_guarantee}] 
Suppose Assumption \ref{ass:parameter_class}, \ref{ass:objective}, \& \ref{ass:loss} hold. There is an absolute constant $C$ such that after round $n$ in Phase I of Algorithm \ref{alg:go_ucb} where $n$ satisfies
\begin{align*}
n \geq C d_w F^2\iota \cdot \max \left\{   \frac{\mu^{\gamma/(2-\gamma)}}{\tau^{2/(2-\gamma)}}, \frac{\zeta}{\mu c^2} \right\},
\end{align*}
with probability $> 1-\delta/2$, regression oracle estimated $\hat{w}_0$ satisfies
\begin{align*}
\|\hat{w}_0 - w^*\|^2_2 \leq \frac{C d_w F^2 \iota}{\mu n},
\end{align*}
where $\iota$ is the logarithmic term depending on $n, C_h, 2/\delta$.
\end{theorem}
\begin{proof}
Recall the definition of expected loss function $L(w) = \E_{x \sim \cU}(f_x(w)-f_x(w^*))^2$ and the second order Taylor's theorem, $L(\hat{w}_0)$ at $w^*$ can be written as
\begin{align*}
L(\hat{w}_0) &= L(w^*) + (\hat{w}_0 - w^*) \nabla L(w^*) + \half \|\hat{w}_0 - w^*\|^2_{\nabla^2 L(\tilde{w})},
\end{align*}
where $\tilde{w}$ lies between $\hat{w}_0$ and $w^*$. Also, because $\nabla L(w^*) = \nabla E_{x\sim \cU}(f_x(w^*) - f_x(w^*))^2=0$, then with probability $> 1-\delta/2$,
\begin{align}
\half \|\hat{w}_0 - w^*\|^2_{\nabla^2 L(\tilde{w})} = L(\hat{w}_0) - L(w^*) \leq \frac{C' d_w F^2 \iota}{n},\label{eq:ll}
\end{align}
where the inequality is due to Lemma \ref{lem:mle_oracle}.

Next, we prove the following lemma stating after a certain number of $n$ samples, $\|\hat{w}_0 - w^*\|_{\nabla^2 L(w^*)}$ can be bounded by the parameter $c$ from our local-self-concordance assumption. 
\begin{lemma}\label{lem:n_hessian}
Suppose Assumption \ref{ass:parameter_class}, \ref{ass:objective}, \& \ref{ass:loss} hold. There is an absolute constant $C'$ such that after round $n$ in Phase I of Algorithm \ref{alg:go_ucb} where $n$ satisfies
\begin{align*}
n \geq 2C' d_w F^2\iota \cdot \max \left\{   \frac{\mu^{\gamma/(2-\gamma)}}{\tau^{2/(2-\gamma)}}, \frac{\zeta}{\mu c^2} \right\},
\end{align*}
then with probability $> 1-\delta/2$, 
\begin{align*}
\|\hat{w}_0 - w^*\|_{\nabla^2 L(w^*)} \leq c. 
\end{align*}
\end{lemma}
\begin{proof}
First we will prove that when $n$ satisfies the first condition, then $\|\hat{w}_0 - w^*\|_2 \leq (\tau/\mu)^{1/(2-\gamma)}$ by a proof by contradiction.

Assume $\|\hat{w}_0 - w^*\|_2 > (\tau/\mu)^{1/(2-\gamma)}$. Check that under this condition, we have $\frac{\tau}{2}\|\hat{w}_0-w^*\|_2^\gamma < \frac{\mu}{2} \|\hat{w}_0-w^*\|_2^2$, therefore the growth-condition
(rather than the local strong convexity) part of the Assumption \ref{ass:loss} is active. By the $(\tau,\gamma)$-growth condition, we have
\begin{align*}
\frac{\tau}{2}\|\hat{w}_0-w^*\|^\gamma_2 &\leq L(\hat{w}_0)-L(w^*) \leq \frac{C' d_w F^2 \iota}{n}.
\end{align*}

Substituting the first lower bound of $n$ in the assumption, we get 
$$\|\hat{w}_0-w^*\|\leq (\tau/\mu)^{1/(2-\gamma)},$$ thus having a contradiction. This proves that when $n$ satisfies the first condition, $\hat{w}_0$ is within the region where local strong convexity is active.

By the local strong-convexity condition, 
\begin{align*}
\frac{\mu}{2} \|\hat{w}_0-w^*\|^2_2 &\leq L(\hat{w}_0)-L(w^*) \leq \frac{C' d_w F^2 \iota}{n }.
\end{align*}
Then,
\begin{align*}
\|\hat{w}_0-w^*\|_{\nabla^2 L(w^*)}&\leq \sqrt{\zeta} \|\hat{w}_0-w^*\|_2 \leq \sqrt{\frac{2\zeta C' d_w F^2 \iota}{\mu n}}.
\end{align*}
Substitute the second lower bound on $n$ that we assumed, we get that
\begin{align*}
\|\hat{w}_0-w^*\|_{\nabla^2 L(w^*)}&\leq  \sqrt{\frac{2\zeta C' d_w F^2 \iota}{\mu n}} \leq c.
\end{align*}
\end{proof}

Now we continue the proof of Theorem \ref{thm:mle_guarantee}. Observe that $\|\tilde{w}- w^*\|_{\nabla^2 L(w^*)}\leq \|\hat{w}_0 - w^*\|_{\nabla^2 L(w^*)}\leq c$, since $\tilde{w}$ lies on the line-segment between  $\hat{w}_0$ and $w^*$. It follows that by the $c$-local self-concordance assumption (Assumption \ref{ass:loss}),
\begin{align*}
(1-c)^2 \|\hat{w}_0- w^*\|^2_{\nabla^2 L(w^*)}\leq \|\hat{w}_0- w^*\|^2_{\nabla^2 L(\tilde{w})}.\label{eq:self}
\end{align*}
Therefore, by eq. \eqref{eq:ll}
\begin{align*}
\|\hat{w}_0 - w^*\|^2_{\nabla^2 L(w^*)} \leq \frac{2C' d_w F^2 \iota }{(1-c)^2n}.
\end{align*}
The proof completes by inequality $\|\hat{w}_0 - w^*\|^2_2 \leq \|\hat{w}_0 - w^*\|^2_{\nabla^2 L(w^*)}/\mu$ due to $\mu$-strongly convexity of $L(w)$ at $w^*$ (Assumption \ref{ass:loss}) and defining $C=2C'/(1-c)^2$.
\end{proof}

\subsection{Properties of Covariance Matrix $\Sigma_t$}\label{sec:sigma}
In eq. \eqref{eq:sigma_t}, $\Sigma_t$ is defined as $\lambda I + \sum_{i=0}^{t-1} \nabla f_{x_i}(\hat{w}_i) \nabla f_{x_i}(\hat{w}_i)^\top$. In this section, we prove three lemmas saying the change of $\Sigma_t$ as $t \in 1,...,T$ is bounded in Phase II of GO-UCB. The key observation is that at each round $i$, the change made to $\Sigma_t$ is $\nabla f_{x_i}(\hat{w}_i) \nabla f_{x_i}(\hat{w}_i)^\top$, which is only rank one. 

\begin{lemma}[Adapted from \citet{agarwal2021rl}]\label{lem:det}
Set $\Sigma_t, \hat{w}_t$ as in eq. \eqref{eq:sigma_t} \& \eqref{eq:inner}, suppose Assumption \ref{ass:parameter_class} \& \ref{ass:loss} hold, and define $u_t = \|\nabla f_{x_t}(\hat{w}_t)\|_{\Sigma^{-1}_t}$. Then
\begin{align*}
\det \Sigma_t = \det \Sigma_0 \prod_{i=0}^{t-1} (1 + u^2_i).
\end{align*}
\end{lemma}
\begin{proof}
Recall the definition of $\Sigma_t = \lambda I + \sum_{i=0}^{t-1} \nabla f_{x_i}(\hat{w}_i) \nabla f_{x_i}(\hat{w}_i)^\top$ and we can show that
\begin{align*}
\det \Sigma_{t+1} &= \det (\Sigma_t + \nabla f_{x_t}(w_t) \nabla f_{x_t}(w_t)^\top)\\
&= \det (\Sigma^\half_t(I + \Sigma^{-\half}_t \nabla f_{x_t}(w_t) \nabla f_{x_t}(w_t)^\top \Sigma^{-\half}_t)\Sigma^\half_t)\\
&= \det (\Sigma_t) \det(I + \Sigma^{-\half}_t \nabla f_{x_t}(w_t) (\Sigma^{-\half}_t \nabla f_{x_t}(w_t))^\top)\\
&= \det (\Sigma_t) \det(I + v_t v_t^\top),
\end{align*}
where $v_t = \Sigma^{-\half}_t \nabla f_{x_t}(w_t)$. Recall $u_t$ is defined as $\|\nabla f_{x_t}(\hat{w}_t)\|_{\Sigma_t^{-1}}$. Because $v_t v^\top_t$ is a rank one matrix, $\det(I + v_t v^\top_t) = 1 + u^2_t$. The proof completes by induction.
\end{proof}

\begin{lemma}[Adapted from \citet{agarwal2021rl}]\label{lem:log_det}
Set $\Sigma_t$ as in eq. \eqref{eq:sigma_t} and suppose Assumption \ref{ass:parameter_class}, \ref{ass:objective}, \& \ref{ass:loss} hold. Then
\begin{align*}
\log \left(\frac{\det \Sigma_{t-1}}{\det \Sigma_0}\right) \leq d_w \log \left(1 + \frac{t C_g^2}{d_w \lambda}\right).
\end{align*}
\end{lemma}
Proof of Lemma \ref{lem:det} directly follows definition of $\Sigma_t$ and proof of Lemma \ref{lem:log_det} involves Lemma \ref{lem:det} and inequality of arithmetic and geometric means. Note $C_g$ is a constant coming from Assumption \ref{ass:objective}. We do not claim any novelty in proofs of these two lemmas which replace feature vector in linear bandit \citep{agarwal2021rl} with gradient vectors.
\begin{proof}
Let $\xi_1,...,\xi_{d_w}$ denote eigenvalues of $\sum_{i=0}^{t-1} \nabla f_{x_i}(w_i) \nabla f_{x_i}(w_i)^\top$, then
\begin{equation}
\sum_{k=1}^{d_w} \xi_k = \tr \left(\sum_{i=0}^{t-1} \nabla f_{x_i}(w_i) \nabla f_{x_i}(w_i)^\top \right) = \sum_{i=0}^{t-1} \|\nabla f_{x_i}(w_i)\|^2_2 \leq t C_g^2, \label{eq:t_c_g}
\end{equation}
where the inequality is by Assumption \ref{ass:objective}. By Lemma \ref{lem:det},
\begin{align*}
\log \left(\frac{\det \Sigma_{t-1}}{\det \Sigma_0}\right) &\leq \log \det \left(I + \frac{1}{\lambda} \sum_{i=0}^{t-1} \nabla f_{x_i}(w_i) \nabla f_{x_i}(w_i)^\top \right)\\
&= \log \left(\prod_{k=1}^{d_w} (1 + \xi_k/\lambda)\right)\\
&= d_w \log \left(\prod_{k=1}^{d_w} (1 + \xi_k/\lambda)\right)^{1/{d_w}}\\
&\leq d_w \log \left(\frac{1}{d_w} \sum_{k=1}^{d_w} (1 + \xi_k/\lambda)\right)\\
&\leq d_w \log \left(1+ \frac{t C_g^2}{d_w \lambda}\right),
\end{align*}
where the second inequality is by inequality of arithmetic and geometric means and the last inequality is due to eq. \eqref{eq:t_c_g}.
\end{proof}

\begin{lemma}\label{lem:sum_of_square}
Set $\Sigma_t, \hat{w}_t$ as in eq. \eqref{eq:sigma_t} \& \eqref{eq:inner} and suppose Assumption \ref{ass:parameter_class}, \ref{ass:objective}, \& \ref{ass:loss} hold. Then
\begin{align*}
\sum_{i=0}^{t-1} \nabla f_{x_i}(\hat{w}_i)^\top \Sigma^{-1}_t \nabla f_{x_i}(\hat{w}_i) \leq 2d_w \log\left(1 + \frac{tC^2_g}{d_w \lambda}\right).
\end{align*}
\end{lemma}

A trivial bound of LHS in Lemma \ref{lem:sum_of_square} could be simply $O(t C^2_g/\lambda)$. Lemma \ref{lem:sum_of_square} is important because it saves the upper bound to be $O(\log(t C^2_g/\lambda))$, which allows us to build a feasible parameter uncertainty ball, shown in the next section.

\begin{proof}
First, we prove $\forall i \in \{0, 1,..., t-1\}, 0 < \nabla f_{x_i}(\hat{w}_i)^\top \Sigma^{-1}_t \nabla f_{x_i}(\hat{w}_i) < 1$. Recall the definition of $\Sigma_t$, it's easy to see that $\Sigma_t$ is a positive definite matrix and thus $0 < \nabla f_{x_i}(\hat{w}_i)^\top \Sigma^{-1}_t \nabla f_{x_i}(\hat{w}_i)$. To prove it's smaller than $1$, we need to decompose $\Sigma_t$ and write 
\begin{align*}
&\quad\ \nabla f_{x_i}(\hat{w}_i)^\top \Sigma^{-1}_t \nabla f_{x_i}(\hat{w}_i)\\
&= \nabla f_{x_i}(\hat{w}_i)^\top \left(\lambda I + \sum_{i=0}^{t-1} \nabla f_{x_i}(\hat{w}_i) \nabla f_{x_i}(\hat{w}_i)^\top \right)^{-1} \nabla f_{x_i}(\hat{w}_i)\\
&= \nabla f_{x_i}(\hat{w}_i)^\top \left(\nabla f_{x_i}(\hat{w}_i) \nabla f_{x_i}(\hat{w}_i)^\top - \nabla f_{x_i}(\hat{w}_i) \nabla f_{x_i}(\hat{w}_i)^\top + \lambda I + \sum_{i=0}^{t-1} \nabla f_{x_i}(\hat{w}_i) \nabla f_{x_i}(\hat{w}_i)^\top \right)^{-1} \nabla f_{x_i}(\hat{w}_i).
\end{align*}
Let $A = - \nabla f_{x_i}(\hat{w}_i) \nabla f_{x_i}(\hat{w}_i)^\top + \lambda I + \sum_{i=0}^{t-1} \nabla f_{x_i}(\hat{w}_i) \nabla f_{x_i}(\hat{w}_i)^\top$, and it becomes
\begin{align*}
\nabla f_{x_i}(\hat{w}_i)^\top \Sigma^{-1}_t \nabla f_{x_i}(\hat{w}_i) = \nabla f_{x_i}(\hat{w}_i)^\top (\nabla f_{x_i}(\hat{w}_i) \nabla f_{x_i}(\hat{w}_i)^\top + A)^{-1} \nabla f_{x_i}(\hat{w}_i).
\end{align*}
By applying Sherman-Morrison lemma (Lemma \ref{lem:sherman}), we have
\begin{align*}
\nabla f_{x_i}(\hat{w}_i)^\top \Sigma^{-1}_t \nabla f_{x_i}(\hat{w}_i) &= \nabla f_{x_i}(\hat{w}_i)^\top \left(A^{-1} - \frac{A^{-1} \nabla f_{x_i}(\hat{w}_i) \nabla f_{x_i}(\hat{w}_i)^\top A^{-1}}{1+ \nabla f_{x_i}(\hat{w}_i)^\top A^{-1} \nabla f_{x_i}(\hat{w}_i)} \right)\nabla f_{x_i}(\hat{w}_i)\\
&= \nabla f_{x_i}(\hat{w}_i)^\top A^{-1} \nabla f_{x_i}(\hat{w}_i) - \frac{\nabla f_{x_i}(\hat{w}_i)^\top A^{-1} \nabla f_{x_i}(\hat{w}_i) \nabla f_{x_i}(\hat{w}_i)^\top A^{-1} \nabla f_{x_i}(\hat{w}_i)}{1+ \nabla f_{x_i}(\hat{w}_i)^\top A^{-1} \nabla f_{x_i}(\hat{w}_i)}\\
&= \frac{\nabla f_{x_i}(\hat{w}_i)^\top A^{-1} \nabla f_{x_i}(\hat{w}_i)}{1+ \nabla f_{x_i}(\hat{w}_i)^\top A^{-1} \nabla f_{x_i}(\hat{w}_i)} < 1.
\end{align*}
Next, we use the fact that $\forall x \in (0,1), x \leq 2\log(1+ x)$, and we have
\begin{align*}
\sum_{i=0}^{t-1} \nabla f_{x_i}(\hat{w}_i)^\top \Sigma^{-1}_t \nabla f_{x_i}(\hat{w}_i) &\leq \sum_{i=0}^{t-1} 2\log\left( 1+ \nabla f_{x_i}(\hat{w}_i)^\top \Sigma^{-1}_t \nabla f_{x_i}(\hat{w}_i)\right)\\
&\leq 2\log\left(\frac{\det \Sigma_{t-1}}{\det \Sigma_0} \right)\\
&\leq 2d_w \log\left( 1 + \frac{tC^2_g}{d_w\lambda}\right),
\end{align*}
where the last two inequalities are due to Lemma \ref{lem:det} and \ref{lem:log_det}.
\end{proof}

\subsection{Feasibility of $\mathrm{Ball}_t$}

\begin{lemma}[Restatement of Lemma \ref{lem:feasible_ball}]
Set $\Sigma_t, \hat{w}_t$ as in eq. \eqref{eq:sigma_t}, \eqref{eq:inner}. Set $\beta_t$ as
\begin{align*}
\beta_t = \tilde{O} \left(d_w \sigma^2 + \frac{d^3_w}{\mu^2} + \frac{d^3_w t }{\mu^2 T}\right).
\end{align*}
Suppose Assumption \ref{ass:parameter_class}, \ref{ass:objective}, \& \ref{ass:loss} hold and choose $n= \sqrt{T}, \lambda = C_\lambda \sqrt{T}$. Then $\forall t \in [T]$ in Phase II of Algorithm \ref{alg:go_ucb}, w.p. $>1-\delta$,
\begin{align*}
\|\hat{w}_t - w^*\|^2_{\Sigma_t} &\leq \beta_t.
\end{align*}
\end{lemma}
\begin{proof}
The proof has three steps. First we obtain the closed form solution of $\hat{w}_t$. Next we derive the upper bound of $\|\hat{w}_i - w^*\|^2_2$. Finally we use it to prove that the upper bound of $\|\hat{w}_t - w^*\|^2_{\Sigma_t}$ matches our choice of $\beta_t$.

\textbf{Step 1: Closed form solution of $\hat{w}_t$.}
The optimal criterion for the objective function in eq. \eqref{eq:opt_inner} is
\begin{align*}
0= \lambda (\hat{w}_t - \hat{w}_0) + \sum_{i=0}^{t-1} ((\hat{w}_t - \hat{w}_i)^\top \nabla f_{x_i}(\hat{w}_i) + f_{x_i}(\hat{w}_i) - y_i) \nabla f_{x_i}(\hat{w}_i).
\end{align*}
Rearrange the equation and we have
\begin{align*}
\lambda (\hat{w}_t- \hat{w}_0) + \sum_{i=0}^{t-1} (\hat{w}_t - \hat{w}_i)^\top \nabla f_{x_i}(\hat{w}_i) \nabla f_{x_i}(\hat{w}_i) &= \sum_{i=0}^{t-1} (y_i - f_{x_i}(\hat{w}_i) ) \nabla f_{x_i}(\hat{w}_i),\\
\lambda (\hat{w}_t- \hat{w}_0) + \sum_{i=0}^{t-1} (\hat{w}_t - \hat{w}_i)^\top \nabla f_{x_i}(\hat{w}_i) \nabla f_{x_i}(\hat{w}_i) &= \sum_{i=0}^{t-1} (y_i - f_{x_i}(w^*) + f_{x_i}(w^*) - f_{x_i}(\hat{w}_i) ) \nabla f_{x_i}(\hat{w}_i),\\
\lambda (\hat{w}_t - \hat{w}_0) + \sum_{i=0}^{t-1} \hat{w}^\top_t \nabla f_{x_i}(\hat{w}_i) \nabla f_{x_i}(\hat{w}_i) &= \sum_{i=0}^{t-1} (\hat{w}^\top_i \nabla f_{x_i}(\hat{w}_i) +\eta_i + f_{x_i}(w^*) - f_{x_i}(\hat{w}_i) ) \nabla f_{x_i}(\hat{w}_i),\\
\hat{w}_t\left(\lambda I + \sum_{i=1}^{t-1} \nabla f_{x_i}(\hat{w}_i) \nabla f_{x_i}(\hat{w}_i)^\top \right) - \lambda \hat{w}_0 &= \sum_{i=0}^{t-1} (\hat{w}^\top_i \nabla f_{x_i}(\hat{w}_i) + \eta_i + f_{x_i}(w^*) - f_{x_i}(\hat{w}_i))\nabla f_{x_i}(\hat{w}_i),\\
\hat{w}_t \Sigma_t &= \lambda \hat{w}_0 + \sum_{i=0}^{t-1} (\hat{w}^\top_i \nabla f_{x_i}(\hat{w}_i) + \eta_i + f_{x_i}(w^*) - f_{x_i}(\hat{w}_i))\nabla f_{x_i}(\hat{w}_i),
\end{align*}
where the second line is by removing and adding back $f_{x_i}(w^*)$, the third line is due to definition of observation noise $\eta$ and the last line is by our choice of $\Sigma_t$ (eq. \eqref{eq:sigma_t}). Now we have the closed form solution of $\hat{w}_t$:
\begin{align*}
\hat{w}_t = \Sigma^{-1}_t \left( \lambda \hat{w}_0 + \sum_{i=0}^{t-1} (\hat{w}^\top_i \nabla f_{x_i}(\hat{w}_i) + \eta_i + f_{x_i}(w^*) - f_{x_i}(\hat{w}_i))\nabla f_{x_i}(\hat{w}_i)\right).
\end{align*}
Further, $\hat{w}_t - w^*$ can be written as
\begin{align}
\hat{w}_t - w^* &= \Sigma^{-1}_t \left(\sum_{i=0}^{t-1} \nabla f_{x_i}(\hat{w}_i) (\nabla f_{x_i}(\hat{w}_i)^\top \hat{w}_i +\eta_i + f_{x_i}(w^*) - f_{x_i}(\hat{w}_i)) \right) + \lambda \Sigma^{-1}_t \hat{w}_0 - \Sigma^{-1}_t \Sigma_t w^* \nonumber \\
&= \Sigma^{-1}_t \left(\sum_{i=0}^{t-1} \nabla f_{x_i}(\hat{w}_i) (\nabla f_{x_i}(\hat{w}_i)^\top \hat{w}_i +\eta_i + f_{x_i}(w^*) - f_{x_i}(\hat{w}_i)) \right) + \lambda \Sigma^{-1}_t (\hat{w}_0 - w^*)\nonumber\\
&\qquad - \Sigma^{-1}_t \left(\sum_{i=0}^{t-1} \nabla f_{x_i}(\hat{w}_i) \nabla f_{x_i}(\hat{w}_i)^\top\right) w^* \nonumber \\
&= \Sigma^{-1}_t \left(\sum_{i=0}^{t-1} \nabla f_{x_i}(\hat{w}_i) (\nabla f_{x_i}(\hat{w}_i)^\top (\hat{w}_i - w^*) +\eta_i + f_{x_i}(w^*) - f_{x_i}(\hat{w}_i)) \right) + \lambda \Sigma^{-1}_t (\hat{w}_0- w^*) \nonumber \\
&= \Sigma^{-1}_t \left(\sum_{i=0}^{t-1} \nabla f_{x_i}(\hat{w}_i) \half \|w^* - \hat{w}_i\|^2_{\nabla^2 f_{x_i}(\tilde{w})}\right) + \Sigma^{-1}_t \left(\sum_{i=0}^{t-1} \nabla f_{x_i}(\hat{w}_i) \eta_i \right) +\lambda \Sigma^{-1}_t (\hat{w}_0 - w^*),\label{eq:w_t_w_star}
\end{align}
where the second line is again by our choice of $\Sigma_t$ and the last equation is by the second order Taylor's theorem of $f_{x_i}(w^*)$ at $\hat{w}_i$ where $\tilde{w}$ lies between $w^*$ and $\hat{w}_i$. 

\textbf{Step 2: Upper bound of $\|\hat{w}_i - w^*\|^2_2$.} Note eq. \eqref{eq:w_t_w_star} holds $\forall i \in [T]$ because all $\hat{w}_i$ are obtained through the same optimization problem, which means
\begin{align*}
\hat{w}_i - w^* = \Sigma^{-1}_i \left(\sum_{\rho=0}^{i-1} \nabla f_{x_\rho}(\hat{w}_\rho) \half \|w^* - \hat{w}_\rho\|^2_{\nabla^2 f_{x_\rho}(\tilde{w})}\right) + \Sigma^{-1}_i \left(\sum_{\rho=0}^{i-1} \nabla f_{x_\rho}(\hat{w}_\rho) \eta_\rho \right) +\lambda \Sigma^{-1}_i (\hat{w}_0 - w^*).
\end{align*}
By inequality $(a+b+c)^2 \leq 4a^2 + 4b^2 + 4c^2$ and definition of $\Sigma_i$, we take the square of both sides and get
\begin{align}
\|\hat{w}_i - w^*\|^2_2 &\leq \frac{4}{\lambda}\left\|\sum_{\rho=0}^{i-1} \nabla f_{x_\rho}(\hat{w}_\rho) \eta_\rho \right\|^2_{\Sigma^{-1}_i} + 4\|\hat{w}_0 - w^*\|^2_2 + \frac{1}{\lambda}\left\|\sum_{\rho=0}^{i-1} \nabla f_{x_\rho}(\hat{w}_\rho) \|w^* - \hat{w}_\rho\|^2_{\nabla^2 f_{x_\rho}(\tilde{w}_\rho)} \right\|^2_{\Sigma^{-1}_i}.\label{eq:w_i}
\end{align}

Now we use induction to prove the convergence rate of $\|\hat{w}_i - w^*\|^2_2, \forall i \in [T]$. Recall at the very beginning of Phase II, by Theorem \ref{thm:mle_guarantee} (check that the condition on $n$ is satisfied due to our condition on $T$ and the choice of $n = \sqrt{T}$),  with probability $> 1- \delta/2$,
\begin{align*}
\|\hat{w}_0 - w^*\|^2_2 \leq \frac{C d_w F^2 \iota}{\mu n}.
\end{align*}

To derive a claim based on induction, formally, we suppose at round $i$, there exists some universal constant $\tilde{C}$ such that with probability $> 1- \delta/2$,
\begin{align*}
\|\hat{w}_i - w^*\|^2_2 &\leq \frac{\tilde{C} d_w F^2 \iota}{ \mu n}.
\end{align*}
Our task is to prove that at round $i+1$ with probability $> 1- \delta/2$,
\begin{align*}
\|\hat{w}_{i+1} - w^*\|^2_2 &\leq \frac{\tilde{C} d_w F^2 \iota}{ \mu n}.
\end{align*}
Note $\tilde{C}$ is for induction purpose, which can be different from $C$.

From eq. \eqref{eq:w_i}, at round $i+1$ we can write
\begin{align*}
\|\hat{w}_{i+1} - w^*\|^2_2 &\leq \frac{4\sigma^2}{\lambda} \log\left(\frac{\det(\Sigma_i)\det(\Sigma_0)^{-1}}{\delta^2_i} \right) +  \frac{4Cd_w F^2 \iota}{\mu n } + \frac{1}{\lambda}\left\|\sum_{\rho=0}^{i} \nabla f_{x_\rho}(\hat{w}_\rho) \|w^* - \hat{w}_\rho\|^2_{\nabla^2 f_{x_\rho}(\tilde{w}_\rho)} \right\|^2_{\Sigma^{-1}_{i+1}}\\
&\leq \frac{4\sigma^2}{\lambda} \left(d_w \log \left(1+\frac{i C_g^2}{d_w \lambda} \right) + \log\left(\frac{\pi^2 i^2}{3\delta}\right) \right) + \frac{4Cd_w F^2 \iota}{\mu n } \nonumber\\
&\qquad + \frac{1}{\lambda}\left\|\sum_{\rho=0}^{i} \nabla f_{x_\rho}(\hat{w}_\rho) \|w^* - \hat{w}_\rho\|^2_{\nabla^2 f_{x_\rho}(\tilde{w}_\rho)} \right\|^2_{\Sigma^{-1}_{i+1}}\\
&\leq \frac{4d_w \sigma^2 \iota'}{\lambda} + \frac{4Cd_w F^2 \iota}{\mu n } + \frac{1}{\lambda}\left\|\sum_{\rho=0}^{i} \nabla f_{x_\rho}(\hat{w}_\rho) \|w^* - \hat{w}_\rho\|^2_{\nabla^2 f_{x_\rho}(\tilde{w}_\rho)} \right\|^2_{\Sigma^{-1}_{i+1}},
\end{align*}
where the first inequality is due to self-normalized bound for vector-valued martingales (Lemma \ref{lem:self_norm} in Appendix \ref{sec:auxiliary}) and Theorem \ref{thm:mle_guarantee}, the second inequality is by Lemma \ref{lem:log_det} and our choice of $\delta_i= 3\delta/(\pi^2 i^2)$, and the last inequality is by defining $\iota'$ as the logarithmic term depending on $i, d_w, C_g, 1/\lambda, 2/\delta$ (with probability $> 1- \delta/2$). The choice of $\delta_i$ guarantees the total failure probability over $t$ rounds is no larger than $\delta/2$. Now we use our assumption $\|\hat{w}_i - w^*\|^2_2 \leq \frac{\tilde{C} d_w F^2 \iota}{\mu n}$ to bound the last term.
\begin{align*}
\|\hat{w}_{i+1} - w^*\|^2_2 &\leq \frac{4d_w \sigma^2 \iota'}{\lambda} + \frac{4Cd_w F^2 \iota}{\mu n } + \frac{\tilde{C}^2 C^2_h d_w^2 F^4 \iota^2}{ \mu^2 \lambda n^2}\left(\sum_{\rho=0}^i \sqrt{\nabla f_{x_\rho}(\hat{w}_\rho)^\top \Sigma^{-1}_{i+1} \nabla f_{x_\rho}(\hat{w}_\rho)} \right)^2\\
&\leq \frac{4d_w \sigma^2 \iota'}{\lambda} + \frac{4Cd_w F^2 \iota}{ \mu n } + \frac{\tilde{C}^2 C^2_h d_w^2 F^4 \iota^2}{ \mu^2 \lambda n^2}\left(\sum_{\rho=0}^i 1\right)\left(\sum_{\rho=0}^i \nabla f_{x_\rho}(\hat{w}_\rho)^\top \Sigma^{-1}_{i+1} \nabla f_{x_\rho}(\hat{w}_\rho) \right)\\
&\leq \frac{4d_w \sigma^2 \iota'}{\lambda} + \frac{4Cd_w F^2 \iota}{ \mu n } + \frac{\tilde{C}^2 C^2_h d^3_w F^4 i \iota{''} \iota^2}{ \mu^2 \lambda n^2},
\end{align*}
where the first inequality is due to smoothness of loss function in Assumption \ref{ass:loss} and triangular inequality, the second inequality is by Cauchy-Schwarz inequality, and the last inequality is because of Lemma \ref{lem:sum_of_square} and defining $\iota{''}$ as logarithmic term depending on $i, d_w, C_g, 1/\lambda$.

What we need is that there exists some universal constant $\tilde{C}$ such that
\begin{align*}
\frac{4d_w \sigma^2 \iota'}{\lambda} + \frac{4C d_w F^2 \iota}{\mu n} + \frac{\tilde{C}^2 C^2_h d^3_w F^4 i \iota^2 \iota{''}}{\lambda \mu^2 n^2} \leq \frac{\tilde{C}d_w F^2 \iota}{\mu n}.
\end{align*}
Note the LHS is monotonically increasing w.r.t $i$ so the inequality must hold when $i=T$, i.e.,
\begin{align*}
\frac{4d_w \sigma^2 \iota'}{\lambda} + \frac{4C d_w F^2 \iota}{\mu n} + \frac{\tilde{C}^2 C^2_h d^3_w F^4 T \iota^2 \iota{''}}{\lambda \mu^2 n^2} \leq \frac{\tilde{C}d_w F^2 \iota}{\mu n}.
\end{align*}
Recall the range of our function is $[-F, F]$, given any distribution, the variance $\sigma^2$ can always be upper bounded by $ F^2/4$, so we just need to show that
\begin{align*}
\frac{d_w F^2 \iota'}{\lambda} + \frac{4C d_w F^2 \iota}{\mu n} + \frac{\tilde{C}^2 C^2_h d^3_w F^4 T \iota^2 \iota{''}}{\lambda \mu^2 n^2} &\leq \frac{\tilde{C}d_w F^2 \iota}{\mu n},\\
\mu^2 n^2 \iota' + 4 \lambda \mu n C  \iota + \tilde{C}^2 C^2_h d^2_w F^2 T \iota^2 \iota{''} &\leq \lambda \mu n \tilde{C} \iota,\\
\tilde{C}^2 C^2_h d^2_w F^2 T \iota^2 \iota{''} - \tilde{C}\lambda \mu n \iota + \mu^2 n^2 \iota' + 4\lambda \mu n C \iota &\leq 0,
\end{align*}
where the second and third lines are by rearrangement. A feasible solution on $\tilde{C}$ requires
\begin{align}
\lambda^2 \mu^2 n^2 \iota^2 - 4C^2_h d^2_w F^2 T \iota^2 \iota{''} (\mu^2 n^2 \iota' + 4\lambda \mu n C \iota) &\geq 0, \nonumber\\
\lambda^2 \mu^2 n - 4C^2_h d^2_w F^2 T\iota{''} (\mu^2 n \iota' + 4\lambda \mu C \iota) &\geq 0,\label{eq:b^2-4ac}
\end{align}
where the second line is by rearrangement. 
Substitute our choices of $\lambda = C_\lambda \sqrt{T}, n=\sqrt{T}$ and solve the quadratic inequality for $C_\lambda$; we get that it suffices to choose
\begin{align}
C_\lambda = 4 \sqrt{C^2_h d^2_w F^2 \iota{'} \iota{''} + \frac{16 C^2 C^4_h d^4_w F^4 \iota^2 \iota{''}^2}{\mu^2}} = \tilde{O}\left(\frac{d_w^2}{\mu}\right),\label{eq:c_lambda}
\end{align}
with assumption $d_w > \mu$. Check that $C_\lambda$ depends only logarithmically on $T$ and that it ensures eq. \eqref{eq:b^2-4ac} holds, therefore certifying that a universal constant $\tilde{C}$ exists. Therefore, by induction, we prove that $\forall i \in [T]$ there exists a universal constant $\tilde{C}$ such that with probability $> 1- \delta/2$,
\begin{align*}
\|\hat{w}_i - w^*\|^2_2 \leq \frac{\tilde{C}d_w F^2\iota}{\mu n}.
\end{align*}
With this result, now we are ready to move to \textbf{Step 3}. 

\textbf{Step 3: Upper bound of $\|\hat{w}_t - w^*\|^2_{\Sigma_t}$.}
Multiply both sides of eq. \eqref{eq:w_t_w_star} by $\Sigma^\half_t$ and we have
\begin{align*}
\Sigma^{\half}_t(\hat{w}_t - w^*) &\leq \half \Sigma^{-\half}_t \left(\sum_{i=0}^{t-1} \nabla f_{x_i}(\hat{w}_i)\|w^* - \hat{w}_i\|^2_{\nabla^2 f_{x_i}(\tilde{w})}\right) + \Sigma^{-\half}_t \left(\sum_{i=0}^{t-1} \nabla f_{x_i}(\hat{w}_i) \eta_i \right) + \lambda \Sigma^{-\half}_t (\hat{w}_0 - w^*).
\end{align*}
Take square of both sides and by inequality $(a+b+c)^2\leq 4a^2 + 4b^2 + 4c^2$ we obtain
\begin{align*}
\|\hat{w}_t - w^*\|^2_{\Sigma_t} &\leq 4\left\| \sum_{i=0}^{t-1} \nabla f_{x_i}(\hat{w}_i) \eta_i\right\|^2_{\Sigma_t^{-1}} + 4\lambda^2 \|\hat{w}_0 - w^*\|^2_{\Sigma^{-1}_t} + \left\|\sum_{i=0}^{t-1} \nabla f_{x_i}(\hat{w}_i)\|w^* - \hat{w}_i\|^2_{\nabla^2 f_{x_i}(\tilde{w})}\right\|^2_{\Sigma^{-1}_t}.
\end{align*}
The remaining proof closely follows \textbf{Step 2}, i.e.,
\begin{align*}
\|\hat{w}_t - w^*\|^2_{\Sigma_t} &\leq 4d_w \sigma^2 \iota' + \frac{4\lambda C d_w F^2 \iota}{\mu n} + \frac{\tilde{C}^2 C^2_h d_w^2 F^4 \iota^2}{\mu^2  n^2}\left(\sum_{i=0}^{t-1} \sqrt{\nabla f_{x_i}(\hat{w}_i)^\top \Sigma^{-1}_{t} \nabla f_{x_i}(\hat{w}_i)} \right)^2\\
&\leq 4d_w \sigma^2 \iota' + \frac{4\lambda C d_w F^2 \iota}{\mu n } + \frac{\tilde{C}^2 C^2_h d_w^2 F^4 \iota^2}{\mu^2  n^2}\left(\sum_{i=0}^{t-1} 1\right)\left(\sum_{i=0}^{t-1} \nabla f_{x_i}(\hat{w}_i)^\top \Sigma^{-1}_t \nabla f_{x_i}(\hat{w}_i) \right)\\
&\leq 4d_w \sigma^2 \iota' + \frac{4\lambda Cd_w F^2 \iota}{\mu n } + \frac{\tilde{C}^2 C^2_h d_w^3 F^4 t \iota{''} \iota^2}{\mu^2  n^2}\\
&\leq \tilde{O}\left(d_w \sigma^2 + \frac{d^3_w}{\mu^2} + \frac{d^3_w t }{\mu^2 T}\right),
\end{align*}
where the last inequality is by our choices of $\lambda=C_\lambda \sqrt{T}, n = \sqrt{T}$. Therefore, our choice of
\begin{align*}
\beta_t &= \tilde{O}\left(d_w \sigma^2 + \frac{d^3_w}{\mu^2} + \frac{d^3_w t }{\mu^2 T}\right)
\end{align*}
guarantees that $w^*$ is always contained in $\mathrm{Ball}_t$ with probability $1- \delta$.
\end{proof}

\subsection{Regret Analysis}\label{sec:regret}

\begin{lemma}[Restatement of Lemma \ref{lem:instant_regret}]
Set $\Sigma_t, \hat{w}_t, \beta_t$ as in eq. \eqref{eq:sigma_t}, \eqref{eq:inner}, \& \eqref{eq:beta_t} and suppose Assumption \ref{ass:parameter_class}, \ref{ass:objective}, \& \ref{ass:loss} hold, then with probability $> 1- \delta$,
$w^*$ is contained in $\mathrm{Ball}_t$.
Define $u_t = \|\nabla f_{x_t}(\hat{w}_t)\|_{\Sigma^{-1}_t}$, then $\forall t \in [T]$ in Phase II of Algorithm \ref{alg:go_ucb},
\begin{align*}
r_t \leq 2\sqrt{\beta_t}u_t + \frac{2\beta_t C_h}{\lambda}.
\end{align*}
\end{lemma}
\begin{proof}
By definition of instantaneous regret $r_t$,
\begin{align*}
r_t &= f_{x^*}(w^*) - f_{x_t}(w^*).
\end{align*}
Recall the selection process of $x_t$ and define $\tilde{w} = \argmax_{w \in \mathrm{Ball}_t} f_{x_t}(w)$,
\begin{align*}
r_t \leq f_{x_t}(\tilde{w}) - f_{x_t}(w^*)= (\tilde{w} - w^*)^\top \nabla f_{x_t}(\dot{w}),
\end{align*}
where the equation is by first order Taylor's theorem and $\dot{w}$ lies between $\tilde{w}$ and $w^*$ which means $\dot{w}$ is guaranteed to be in $\mathrm{Ball}_t$ since $\mathrm{Ball}_t$ is convex. Then, by adding and removing terms,
\begin{align*}
r_t &= (\tilde{w} - \hat{w}_t + \hat{w}_t - w^*)^\top (\nabla f_{x_t}(\hat{w}_t) - \nabla f_{x_t}(\hat{w}_t) + \nabla f_{x_t}(\dot{w}))\\
&\leq \|\tilde{w}-\hat{w}_t\|_{\Sigma_t} \|\nabla f_{x_t}(\hat{w}_t)\|_{\Sigma_t^{-1}} + \|\hat{w}_t - w^*\|_{\Sigma_t} \|\nabla f_{x_t}(\hat{w}_t)\|_{\Sigma_t^{-1}} + (\tilde{w} - \hat{w}_t)^\top (\nabla f_{x_t}(\dot{w}_t) - \nabla f_{x_t}(\hat{w}_t))\nonumber\\
&\qquad + (\hat{w}_t - w^*)^\top (\nabla f_{x_t}(\dot{w}) - \nabla f_{x_t}(\hat{w}_t)),
\end{align*}
where the last inequality is due to Holder's inequality. By definitions of $\beta_t$ in $\mathrm{Ball}_t$ and $u_t = \|\nabla f_{x_t}(\hat{w}_t)\|_{\Sigma^{-1}_t}$,
\begin{align*}
r_t &\leq 2\sqrt{\beta_t} u_t + (\tilde{w} - \hat{w}_t)^\top (\nabla f_{x_t}(\dot{w}) - \nabla f_{x_t}(\hat{w}_t)) + (\hat{w}_t - w^*)^\top (\nabla f_{x_t}(\dot{w}) - \nabla f_{x_t}(\hat{w}_t)).
\end{align*}
Again by first order Taylor's theorem where $\ddot{w}$ lies between $\dot{w}$ and $\hat{w}$ and thus $\ddot{w}$ lies in $\mathrm{Ball}_t$,
\begin{align*}
r_t &\leq 2\sqrt{\beta_t} u_t + (\tilde{w}-\hat{w}_t)^\top \Sigma^\half_t \Sigma^{-\half}_t \nabla^2 f_{x_t}(\ddot{w}) \Sigma^{-\half}_t \Sigma^\half_t (\dot{w}-\hat{w}_t)  + (\hat{w}_t- w^*)^\top \Sigma^\half_t \Sigma^{-\half}_t \nabla^2 f_{x_t}(\ddot{w}) \Sigma^{-\half}_t \Sigma^\half_t (\dot{w}-\hat{w}_t)\\
&\leq 2\sqrt{\beta_t} u_t + \|(\tilde{w}-\hat{w}_t)^\top \Sigma^\half_t\|_2  \|\Sigma^{-\half}_t \nabla^2 f_{x_t}(\ddot{w}) \Sigma^{-\half}_t\|_\mathrm{op} \|\Sigma^\half_t (\dot{w}-\hat{w}_t)\|_2 \nonumber\\
&\qquad + \|(\hat{w}_t - w^*)^\top \Sigma^\half_t\|_2 \|\Sigma^{-\half}_t \nabla^2 f_{x_t}(\ddot{w}) \Sigma^{-\half}_t\|_\mathrm{op} \|\Sigma^\half_t (\dot{w}-\hat{w}_t)\|_2\\
&\leq 2\sqrt{\beta_t}u_t + \frac{2\beta_t C_h}{\lambda},
\end{align*}
where the second inequality is by Holder's inequality and the last inequality is due to definition of $\beta_t$ in $\mathrm{Ball}_t$, Assumption \ref{ass:objective}, and our choice of $\Sigma_t$.
\end{proof}

\begin{lemma}[Restatement of Lemma \ref{lem:sos_instant_regret}]
Set $\Sigma_t, \hat{w}_t, \beta_t$ as in eq. \eqref{eq:sigma_t}, \eqref{eq:inner}, \& \eqref{eq:beta_t} and suppose Assumption \ref{ass:parameter_class}, \ref{ass:objective}, \& \ref{ass:loss} hold, then with probability $> 1- \delta$,
$w^*$ is contained in $\mathrm{Ball}_t$ and $\forall t \in [T]$ in Phase II of Algorithm \ref{alg:go_ucb},
\begin{align*}
\sum_{t=1}^T r^2_t \leq 16\beta_T d_w \log \left(1 + \frac{TC_g^2}{d_w \lambda}\right) + \frac{8\beta^2_T C^2_h T}{\lambda^2}.
\end{align*}
\end{lemma}
\begin{proof}
By Lemma \ref{lem:instant_regret} and inequality $(a+b)^2 \leq 2a^2 + 2b^2$,
\begin{align*}
\sum_{t=1}^T r^2_t &\leq \sum_{t=1}^T 8\beta_t u^2_t + \frac{8 \beta^2_t C^2_h}{\lambda^2}\\
&\leq 8\beta_T \sum_{i=1}^T u^2_t + \frac{8\beta^2_T C^2_h T}{\lambda^2}\\
&\leq 16\beta_T d_w \log \left(1 + \frac{TC_g^2}{d_w \lambda}\right) + \frac{8\beta^2_T C^2_h T}{\lambda^2},
\end{align*}
where the second inequality is due to $\beta_t$ is increasing in $t$ and the last inequality is by Lemma \ref{lem:sum_of_square}.
\end{proof}

By putting everything together, we are ready to prove the main cumulative regret theorem.

\begin{proof}[Proof of Theorem \ref{thm:cr}]
By definition of cumulative regret including both Phase I and II,
\begin{align*}
R_{\sqrt{T}+T} &= \sum_{j=1}^{\sqrt{T}} r_j + \sum_{t=1}^T r_t\\
&\leq 2\sqrt{T}F + \sqrt{T\sum_{t=1}^T r^2_t}\\
&\leq 2\sqrt{T}F + \sqrt{16T\beta_T d_w \log \left(1 + \frac{T
C_g^2}{d_w \lambda}\right) + \frac{8T^2\beta^2_T C^2_h}{\lambda^2}}\\
&\leq \tilde{O}\left(\sqrt{T} F + \sqrt{T\beta_T d_w + \frac{T^2 \beta^2_T }{\lambda^2}}\right),
\end{align*}
where the first inequality is due to function range and Cauchy-Schwarz inequality, the second inequality is by Lemma \ref{lem:sos_instant_regret} and the last inequality is obtained by setting $\lambda=C_\lambda \sqrt{T}, n=\sqrt{T}$ as required by Lemma \ref{lem:feasible_ball} where $C_\lambda$ is in eq. \eqref{eq:c_lambda}.

Recall that $\beta_t$ is defined in eq. \eqref{eq:beta_t}, so
\begin{align*}
\beta_T &= \tilde{O}\left(\frac{d^3_w}{\mu^2} \right).
\end{align*}
The proof completes by plugging in upper bound of $\beta_T$.
\end{proof}

\section{Additional Experimental Details}
In addition to Experiments section in main paper, in this section, we show details of algorithm implementation and and real-world experiments.

\subsection{Implementation of GO-UCB}\label{sec:imp_goucb}
Noise parameter $\sigma=0.01$. Regression oracle in GO-UCB is approximated by stochastic gradient descent algorithm on our two linear layer neural network model with mean squared error loss, $2000$ iterations and $10^{-11}$ learning rate. Exactly solving optimization problem in Step 5 of Phase II may not be computationally tractable, so we use iterative gradient ascent algorithm over $x$ and $w$ with $2000$ iterations and $10^{-4}$ learning rate. $\beta_t$ is set as $d^3_w F^4 t/T$. $\lambda$ is set as $\sqrt{T}\log^2 T$.

\subsection{Real-world Experiments}\label{sec:real_detail}

Hyperparameters can be continuous or categorical, however, in order to fairly compare GO-UCB with Bayesian optimization methods, in all hyperparameter tuning tasks, we set function domain to be $[0, 10]^{d_x}$, a continuous domain. If a hyperparameter is categorical, we allocate equal length domain for each hyperparameter. For example, the seventh hyperparameter of random forest is a bool value, True or False and we define $[0, 5)$ as True and $[5,10]$ as False. If a hyperparameter is continuous, we set linear mapping from the hyperparameter domain to $[0,10]$. For example, the sixth hyperparameter of multi-layer perceptron is a float value in $(0,1)$ thus we multiply it by $10$ and map it to $(0, 10)$.

\textbf{Hyperparameters in hyperparameter tuning tasks.} We list hyperparameters in all three tasks as follows.

Classification with Random Forest.
\begin{enumerate}
    \item Number of trees in the forest, (integer, [20, 200]).
    \item Criterion, (string, ``gini'', ``entropy'', or ``logloss'').
    \item Maximum depth of the tree, (integer, [1, 10]).
    \item Minimum number of samples required to split an internal node, (integer, [2, 10]).
    \item Minimum number of samples required to be at a leaf node, (integer, [1, 10]).
    \item Maximum number of features to consider when looking for the best split, (string, ``sqrt'' or ``log2'').
\item Bootstrap, (bool, True or False).
\end{enumerate}
Classification with Multi-Layer Perceptron.
\begin{enumerate}
    \item Activation function (string, ``identity'', ``logistic'', ``tanh'', or ``relu'').
    \item Strength of the L2 regularization term, (float, [$10^{-6}, 10^{-2}$]).
    \item Initial learning rate used, (float, [$10^{-6}, 10^{-2}$]).
    \item Maximum number of iterations, (integer, [100, 300]).
    \item Whether to shuffle samples in each iteration, (bool, True or False).
    \item Exponential decay rate for estimates of first moment vector, (float, (0, 1)).
    \item Exponential decay rate for estimates of second moment vector (float, (0, 1)).
    \item Maximum number of epochs to not meet tolerance improvement, (integer, [1, 10]).
\end{enumerate}
Classification with Gradient Boosting.
\begin{enumerate}
    \item Loss, (string, ``logloss'' or ``exponential'').
\item Learning rate, (float, (0, 1)).
\item Number of estimators, (integer, [20, 200]).
\item Fraction of samples to be used for fitting the individual base learners, (float, (0, 1)).
\item Function to measure the quality of a split, (string, ``friedman mse'' or ``squared error'').
\item Minimum number of samples required to split an internal node, (integer, [2, 10]).
\item Minimum number of samples required to be at a leaf node, (integer, [1, 10]).
\item Minimum weighted fraction of the sum total of weights, (float, (0, 0.5)).
\item Maximum depth of the individual regression estimators, (integer, [1, 10]).
\item Number of features to consider when looking for the best split, (float, ``sqrt'' or ``log2'').
\item Maximum number of leaf nodes in best-first fashion, (integer, [2, 10]).
\end{enumerate}

\begin{figure*}[!htbp]
	\centering
	\begin{minipage}{0.32\linewidth}\centering
		\includegraphics[width=\textwidth]{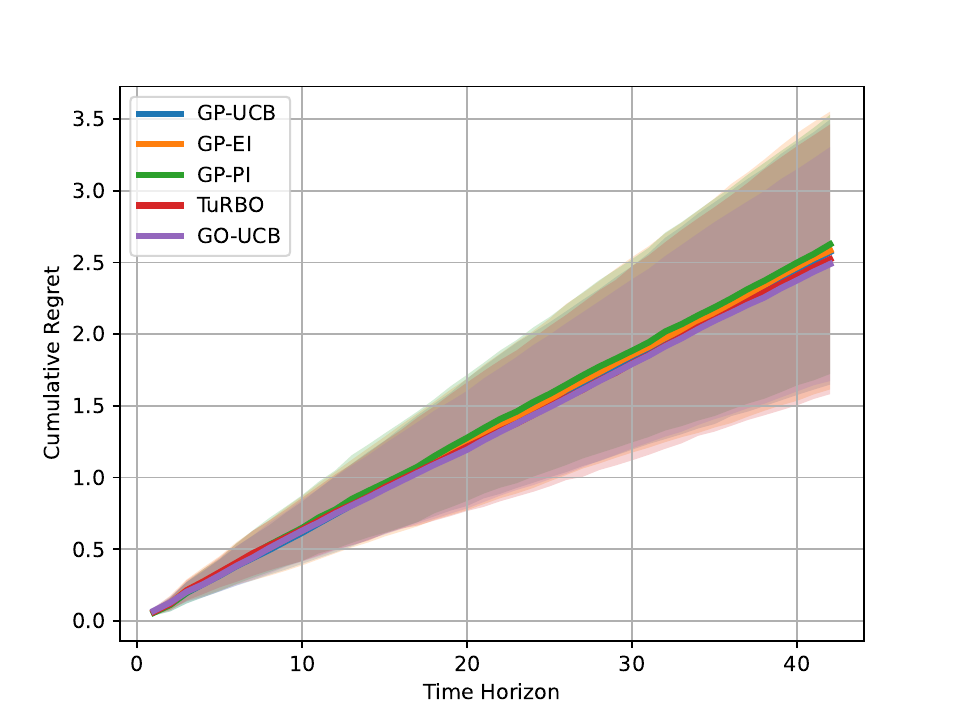}
		(a) Random forest ($d_x = 7$)
	\end{minipage}
	\begin{minipage}{0.32\linewidth}\centering
		\includegraphics[width=\textwidth]{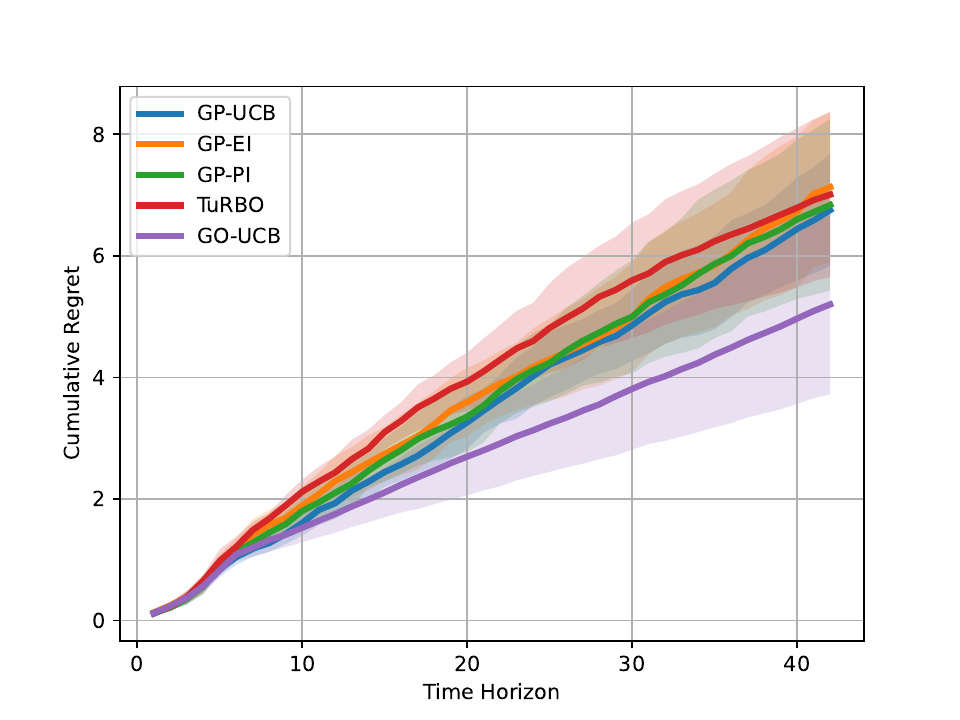}
		(b) Multi-layer perceptron ($d_x = 8$)
	\end{minipage}
	\begin{minipage}{0.32\linewidth}\centering
		\includegraphics[width=\textwidth]{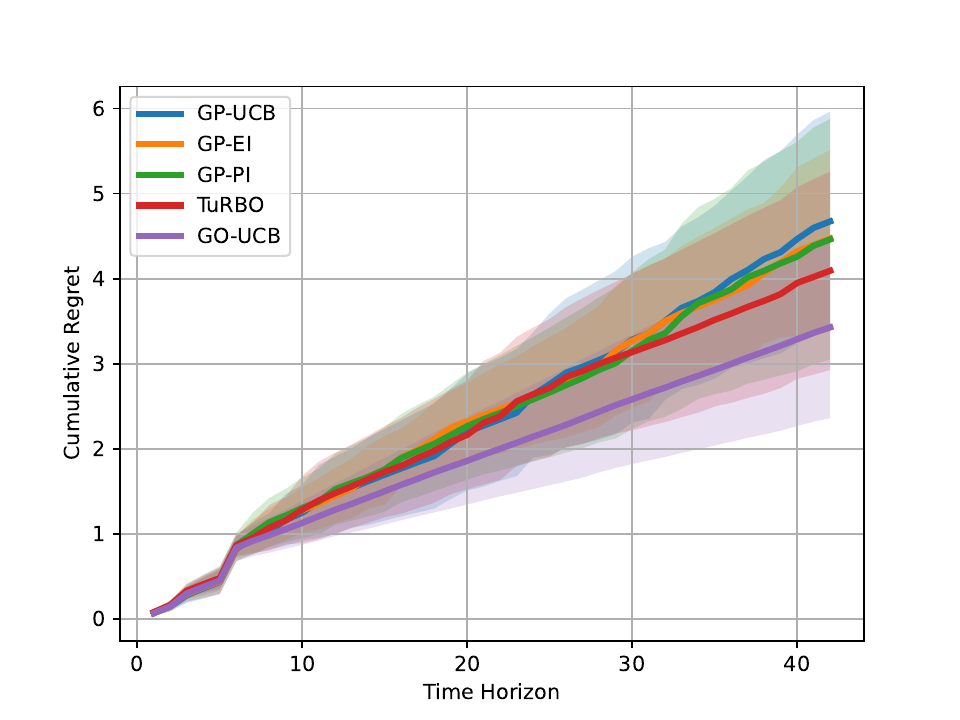}
		(c) Gradient boosting ($d_x = 11$)
	\end{minipage}
\caption{Cumulative regrets (the lower the better) of all algorithms in real-world hyperparameter tuning task on Australian dataset.
}\label{fig:aus}
\end{figure*}

\begin{figure*}[!htbp]
	\centering
	\begin{minipage}{0.32\linewidth}\centering
		\includegraphics[width=\textwidth]{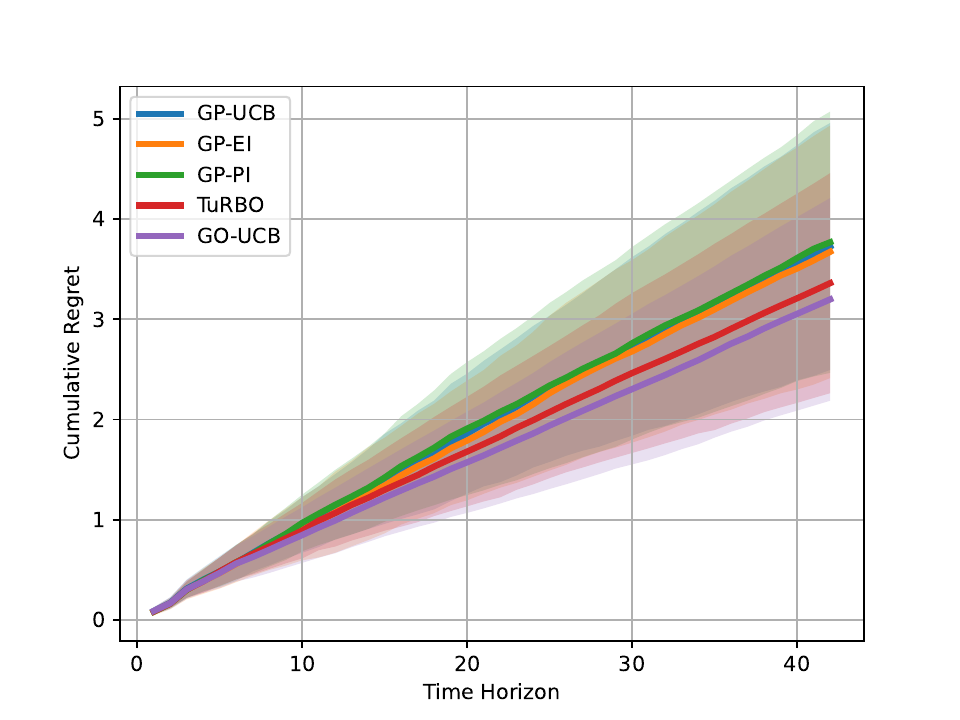}
		(a) Random forest ($d_x=7$)
	\end{minipage}
	\begin{minipage}{0.32\linewidth}\centering
		\includegraphics[width=\textwidth]{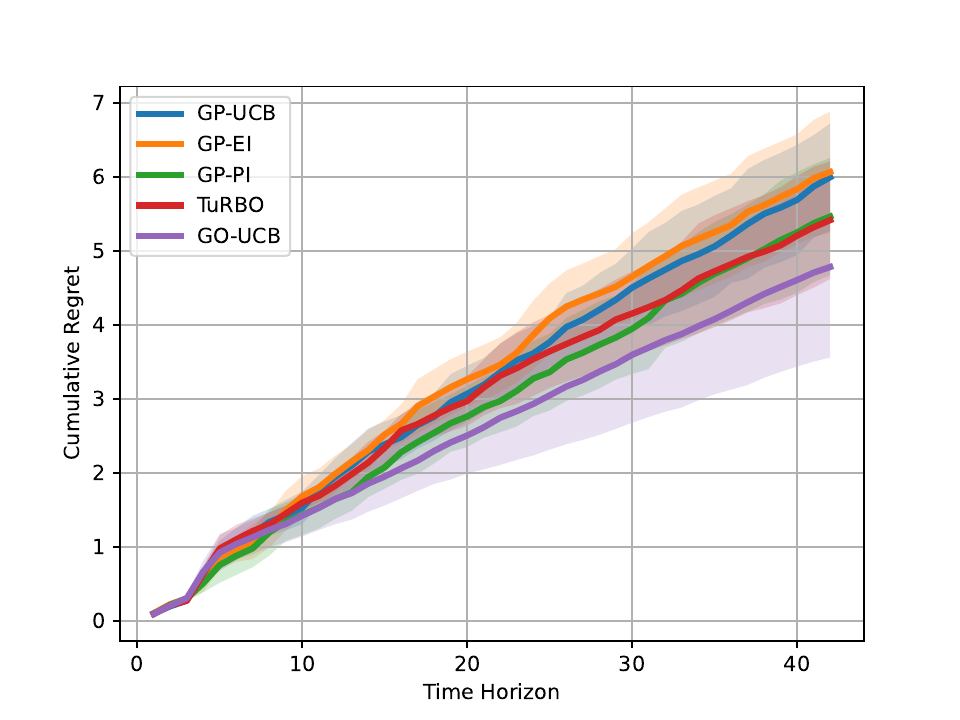}
		(b) Multi-layer perceptron ($d_x=8$)
	\end{minipage}
	\begin{minipage}{0.32\linewidth}\centering
		\includegraphics[width=\textwidth]{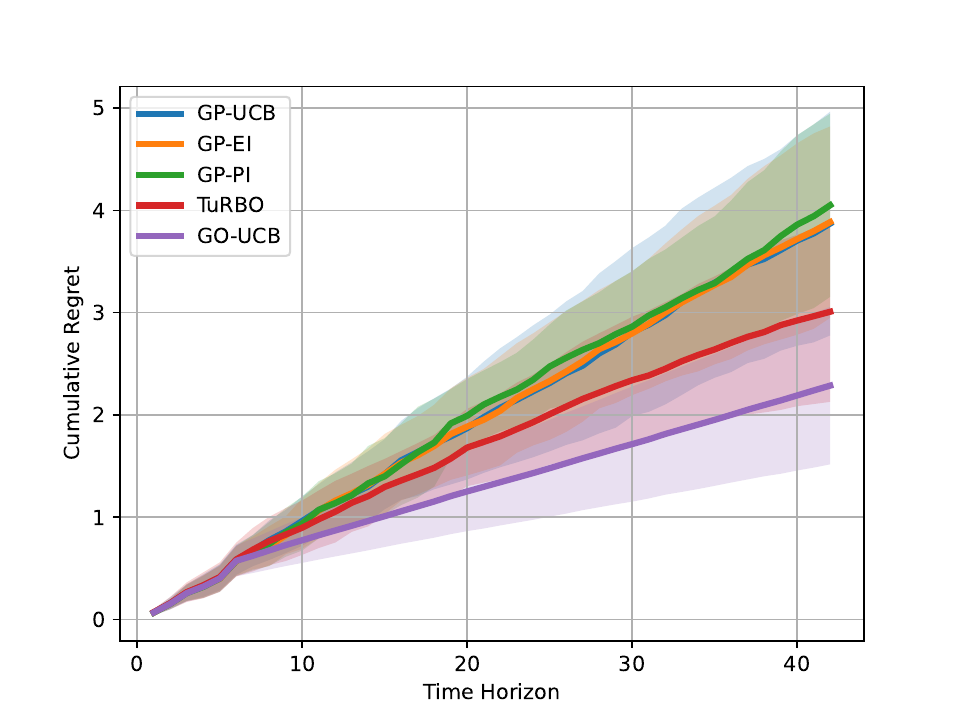}
		(c) Gradient boosting ($d_x=11$)
	\end{minipage}
\caption{Cumulative regrets (the lower the better) of all algorithms in real-world hyperparameter tuning task on Diabetes dataset.
}\label{fig:dia}
\end{figure*}

\textbf{Results on Australian and Diabetes datasets.} Due to page limit of the main paper, we show experimental results of hyperparameter tuning tasks on Australian and Diabetes datasets in Figure \ref{fig:aus} and Figure \ref{fig:dia}. Our proposed GO-UCB algorithm performs consistently better than all other algorithms, which is the same as on Breast-cancer dataset in main paper.

\end{document}